\documentclass[acmtog]{acmart}

\usepackage{cleveref}
\crefname{property}{property}{properties}
\crefname{figure}{Fig.}{Figs.}
\crefname{equation}{Eq.}{Eqs.}
\crefname{table}{Tab.}{Tabs.}
\crefname{section}{Sec.}{Secs.}

\usepackage{xspace}
\usepackage{multirow}
\usepackage{makecell}
\usepackage{subcaption}
\usepackage{colortbl}
\usepackage{enumitem}

\AtBeginDocument{%
  \providecommand\BibTeX{{%
    \normalfont B\kern-0.5em{\scshape i\kern-0.25em b}\kern-0.8em\TeX}}}

\setcopyright{acmlicensed}
\copyrightyear{2024}
\acmYear{2024}
\acmDOI{10.1145/3680528.3687651}
\acmISBN{979-8-4007-1131-2/24/12}

\acmConference[SA Conference Papers '24]{SIGGRAPH Asia 2024 Conference Papers}{December 3--6, 2024}{Tokyo, Japan}
\acmBooktitle{SIGGRAPH Asia 2024 Conference Papers (SA Conference Papers '24), December 3--6, 2024, Tokyo, Japan}

\acmSubmissionID{739}

\begin{document}

\newcommand{\POmega}{\ensuremath{\Sigma}}
\newcommand{\PM}{\ensuremath{\mathrm{M}}}
\newcommand{\cstart}{\ensuremath{\mathbf{s}}}
\newcommand{\cgoal}{\ensuremath{\mathbf{g}}}
\newcommand{\cfst}{\ensuremath{\mathbf{c_i}}}
\newcommand{\csec}{\ensuremath{\mathbf{c_k}}}

\definecolor{mydarkblue}{rgb}{0,0.08,1}
\definecolor{mydarkgreen}{rgb}{0.02,0.6,0.02}
\definecolor{mydarkred}{rgb}{0.8,0.02,0.02}
\definecolor{mydarkorange}{rgb}{0.40,0.2,0.02}
\definecolor{mypurple}{RGB}{111,0,255}
\definecolor{myred}{rgb}{1.0,0.0,0.0}
\definecolor{mygold}{rgb}{0.75,0.6,0.12}
\definecolor{mydarkgray}{rgb}{0.66, 0.66, 0.66}
\definecolor{Brown2}{RGB}{238,59,59}
\definecolor{url_color}{RGB}{42, 83, 163}

\newtheorem{property}[theorem]{Property}

\newcommand{\pcplanner}{PC-Planner\xspace}
\newcommand{\ourmidrulewidth}{0.11em}
\newcommand{\metricsize}{4}


\definecolor{colorFst}{RGB}{191,225,202}       %
\definecolor{colorSnd}{RGB}{227,237,186}     %

\newcommand{\fs}{\cellcolor{colorFst}\bf}   %
\newcommand{\nd}{\cellcolor{colorSnd}}      %

\makeatletter
\DeclareRobustCommand\onedot{\futurelet\@let@token\@onedot}
\def\@onedot{\ifx\@let@token.\else.\null\fi\xspace}

\def\eg{e.g\onedot} 
\def\Eg{E.g\onedot}
\def\ie{i.e\onedot} 
\def\Ie{I.e\onedot}
\def\cf{c.f\onedot} 
\def\Cf{C.f\onedot}
\def\etc{etc\onedot} 
\def\vs{vs\onedot}
\def\wrt{w.r.t\onedot} 
\def\dof{d.o.f\onedot}
\def\etal{et al\onedot}
\makeatother

\title[\pcplanner]{\pcplanner: Physics-Constrained Self-Supervised Learning for Robust Neural Motion Planning with Shape-Aware Distance Function}


\author{Xujie Shen}
\authornote{Xujie Shen and Haocheng Peng contributed equally to this work.}
\email{shenfishcrap@gmail.com}
\affiliation{%
  \institution{Zhejiang University}
  \city{Hangzhou}
  \country{China}
}
\author{Haocheng Peng}
\authornotemark[1]
\email{hchaocheng0223@gmail.com}
\affiliation{%
  \institution{Zhejiang University}
  \city{Hangzhou}
  \country{China}
}

\author{Zesong Yang}
\email{zesongyang0@zju.edu.cn}
\affiliation{%
  \institution{Zhejiang University}
  \city{Hangzhou}
  \country{China}
}

\author{Juzhan Xu}
\email{juzhan.xu@gmail.com}
\affiliation{%
  \institution{Shenzhen University}
  \city{Shenzhen}
  \country{China}
}

\author{Hujun Bao}
\email{bao@cad.zju.edu.cn}
\affiliation{%
 \institution{Zhejiang University}
 \city{Hangzhou}
 \country{China}
}

\author{Ruizhen Hu}
\email{ruizhen.hu@gmail.com}
\affiliation{%
  \institution{Shenzhen University}
  \city{Shenzhen}
  \country{China}
}

\author{Zhaopeng Cui}
\authornote{Corresponding author.}
\email{zhpcui@gmail.com}
\affiliation{%
 \institution{Zhejiang University}
 \city{Hangzhou}
 \country{China}
}

\renewcommand{\shortauthors}{Shen et al.}

\begin{abstract}
  Motion Planning (MP) is a critical challenge in robotics, especially pertinent with the burgeoning interest in embodied artificial intelligence. Traditional MP methods often struggle with high-dimensional complexities. Recently neural motion planners, particularly physics-informed neural planners based on the Eikonal equation, have been proposed to overcome the curse of dimensionality. However, these methods perform poorly in complex scenarios with shaped robots due to multiple solutions inherent in the Eikonal equation.
  To address these issues, this paper presents \pcplanner, a novel physics-constrained self-supervised learning framework for robot motion planning with various shapes in complex environments. To this end, we propose several physical constraints, including monotonic and optimal constraints, to stabilize the training process of the neural network with the Eikonal equation. 
  Additionally, we introduce a novel shape-aware distance field that considers the robot's shape for efficient collision checking and Ground Truth (GT) speed computation. This field reduces the computational intensity, and facilitates adaptive motion planning at test time.
  Experiments in diverse scenarios with different robots demonstrate the superiority of the proposed method in efficiency and robustness for robot motion planning, particularly in complex environments.
  Code and data are available on the project webpage:
  \urlstyle{tt}
  \textcolor{url_color}{\url{https://zju3dv.github.io/pc-planner}}.
\end{abstract}

\ccsdesc[500]{Computing methodologies~Robotic planning}
\ccsdesc[500]{Computing methodologies~Motion path planning}

\keywords{motion planning, robot navigation, Eikonal equation, self-supervised learning}

\begin{teaserfigure}
  \centering
  \includegraphics[width=\textwidth]{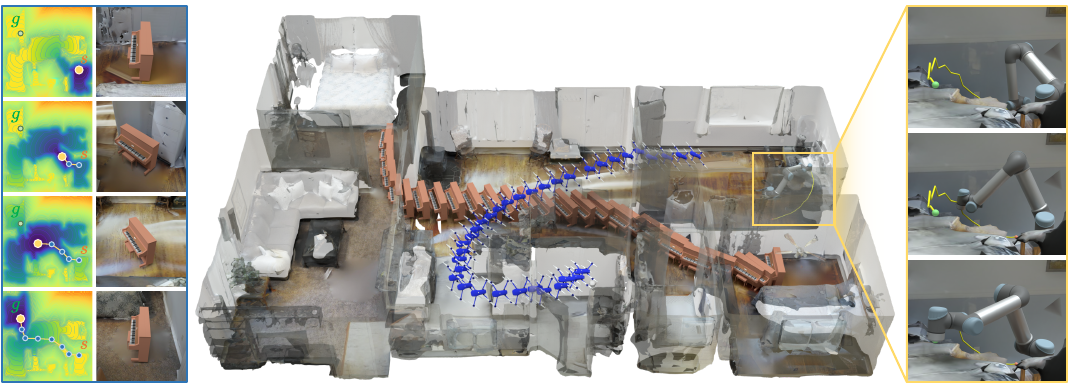}
  \caption{Given a prebuilt complex 3D environment, our \pcplanner can learn the time fields and execute motion planning for robots of various shapes from any start state to any goal state in a self-supervised manner. Left:  Time fields for a piano navigating through the environment, in which the rainbow color scheme is used. Right:  Close-up views of the manipulation of a UR5 robot.
  }
  \Description{robots navigates in 3D environments}
  \label{fig:teaser}
\end{teaserfigure}


\maketitle

\section{Introduction}

Motion planning (MP) is a long-standing problem in robotics that aims to find a trajectory from a start configuration to a goal configuration while satisfying all constraints like collision avoidance. 
With the rapid advancement of embodied artificial intelligence, this problem has garnered increasing attention.
The traditional MP methods normally adopt sampling techniques to explore the robot's obstacle-free state-space and construct feasible paths \cite{kingston2018sampling,karaman2011sampling,gammell2015batch}. However, these methods often encounter limitations in high-dimensional spaces with 
increased computational complexity and reduced efficiency.

Recently learning-based methods \cite{Neural_RRT,li2021learning,chaplot2021differentiable}, \ie, neural motion planners, have been proposed to solve the curse of dimensionality. Most of the learning-based methods \cite{huh2021cost,li2021learning} utilize the deep neural network to predict the motions iteratively to improve the efficiency of planning in high dimensions. However, these methods face some significant limitations. Firstly, these data-driven methods heavily rely on expert training data such as the trajectories from the traditional methods, resulting in time-consuming data generation processes, particularly for high-dimensional spaces. Additionally, they typically employ constant velocity paradigms which are not suitable for real scenarios~\cite{vysocky2019motion}. 
In contrast, some recently proposed physics-informed methods \cite{ni2023ntfields,ni2023progressive} offer a compelling alternative. 
These methods first predefine a speed field that considers velocity constraints based on the geometry of obstacles. Utilizing this predefined speed field, they employ neural networks to solve the Eikonal equation for motion planning. The training data can be efficiently generated by sampling within the speed field, thereby circumventing the need for expert data and significantly reducing data generation time.
However, these physics-informed methods struggle in complex environments with shaped robots due to the multiple solutions inherent in the Eikonal equation.
Consequently, they may produce time fields with local minima, as shown in \cref{fig:time_field} (left), leading to infeasible paths and thus poor performance in complex planning tasks.
Unlike some optimization-based methods, where ``local minima'' refers to feasible but sub-optimal solutions, in our context, local minima in the time fields lead to infeasible solutions, thereby reducing the success rates.

In this paper, we present \pcplanner, a novel physics-constrained self-supervised learning framework for neural robot motion planning based on a new shape-aware distance field, effectively addressing the local minima in the time fields and thus significantly outperforming existing physics-informed methods in success rates.

Based on a deep analysis of the properties of the Eikonal equation, we propose two physical constraints and incorporate them into the training process of the neural network with the Eikonal equation for motion planning in a self-supervised manner.
Specifically, we introduce the \emph{monotonic constraint} to identify the local minima and the \emph{optimal constraint} to preserve the optimality of the solution to the Eikonal equation. These constraints are seamlessly integrated into a self-supervised training framework, allowing the network to recognize situations where the solution is trapped in local minima and enabling it to self-correct by adhering to the defined physical rules, as shown in ~\cref{fig:time_field} (right).

However, these physical constraints are only applied to the collision-free paths, which introduces significant computational overhead for trajectory collision checks during training.
To address this issue, we further propose a novel shape-aware distance field (SADF) that models the minimum distance from the robot with any shape and configuration to the environment. With this neural implicit field, we can efficiently conduct collision checking and apply the physical constraints during the training process of the physics-informed neural model. Meanwhile, our SADF can be efficiently converted into the required speed field for training the physics-informed model, and further exploited during the test stage for collision avoidance with adaptive path planning, 

We analyze and validate the effectiveness of our method through experiments in diverse scenarios with various robots, demonstrating the superiority of \pcplanner over the existing methods.

\begin{figure}[!tb]
    \centering
    \includegraphics[width=\linewidth]{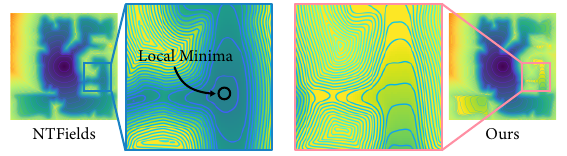}
    \caption{Comparison of time fields for Gibson. NTFields generates an incorrect time field with local minima due to the inherent multiple solutions in the Eikonal equation, while our \pcplanner learns to generate the correct time field with the proposed physical constraints.
    }
    \label{fig:time_field}
    \Description{Time fields comparison.}
\end{figure}

Our main contributions can be summarized as follows:
\begin{itemize}
    \item We introduce a novel physics-constrained self-supervised learning approach for physics-informed neural robot motion planning, which enables efficient and robust motion planning for robots with various shapes in complex scenarios.
    \item We propose two physical constraints to enable the network to jump out of local minima and converge to the correct solutions that obey the physical rules.
    \item We develop a new neural shape-aware distance field for collision checking that can predict the minimum distance to the environment for any robot with arbitrary shapes and configurations in the fixed environment, which facilitates both self-supervised training and test stages.
\end{itemize}

\section{Related Work}
Our work focuses on the problem of robust motion planning for robots with arbitrary shapes in challenging scenarios. Here, we review the current state of research on motion planning and implicit neural distance fields.

\noindent\textbf{Motion Planning.}
Existing optimal path planning methods include sampling-based approaches\cite{lavalle2001randomized,karaman2011sampling,gammell2015batch,tukan2022obstacle} that explore the environment through random state sampling to retrieve feasible paths, optimization-based methods\cite{kalakrishnan2011stomp,kurenkov2022nfomp,mukadam2016gaussian} which minimize a defined cost while meeting constraints to find the optimal trajectory, \etc\cite{yang2019survey}. While effective for high-dimensional tasks, they suffer from high computational costs and unstable solutions due to their sensitivity to initial conditions.
Neural motion planners (NMPs)\cite{Neural_RRT,LSDtransformer,chaplot2021differentiable,ichter2018learning,li2021learning} have emerged to balance efficiency and stability, using expert trajectories from classic methods for training. However, generating training data from traditional methods is computationally expensive, limiting their flexibility.

The most pertinent method to our work is the physics-driven method.
The earliest of these methods is FMM\cite{chopp2001some,treister2016fast,sethian1996fast,pykonal_fmm} which numerically solves the Eikonal equation.
However, its computational complexity increases dramatically with dimensionality.
Recently, NTFields\cite{ni2023ntfields} has introduced a continuous-time path planning method encoding the Eikonal equation directly into the network, eliminating the need for expert trajectory supervision. Subsequently, \cite{ni2023progressive} incorporates a viscosity term and progressive learning for smoother path solutions. Despite their success, these methods struggle with challenging scenarios affected by local minima. Our approach addresses the issues, ensuring high success rates and fast inference across diverse scenarios including high-dimensional planning, and various shapes of robots.

\noindent\textbf{Implicit distance representation.}
The contemporary approach \cite{chabra2020deep,ouasfi2022few,jiang2020local} to learning Signed Distance Fields (SDF) primarily involves employing neural networks to regress the mapping between 3D coordinates to the signed distances. Similar work can be traced back to \cite{park2019deepsdf}, where a neural signed distance function was introduced. The function allows querying the shortest distance between the object surface and continuous spatial points. However, it is designed for point queries, and when dealing with a shaped robot as the target, such distance functions cannot be directly applied. 
Recent works like \cite{chou2022gensdf,zhu2023e2pn,chen2021equivariant} excel in accurately representing the shape of objects, encompassing the capability to handle key properties such as scaling and rotation.
These methods serve as inspiration for our proposed shape-aware distance function, allowing us to characterize the distance between the robot and the environment.
\begin{figure*}
    \centering
    \includegraphics[width=\linewidth]{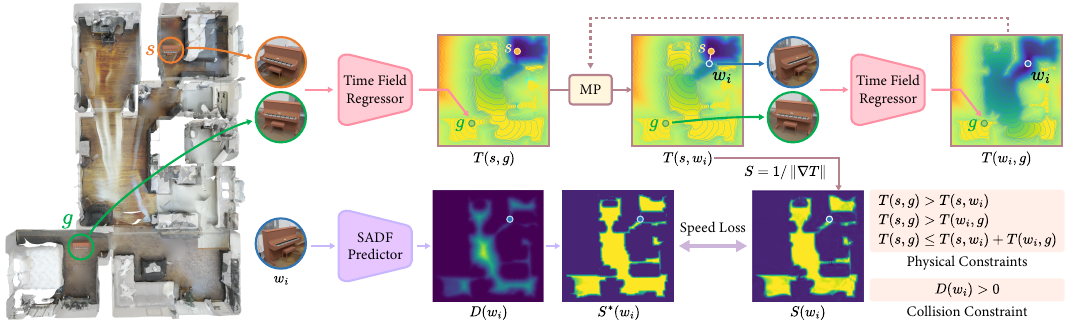}
    \caption{    
    The \pcplanner integrates a physics-constrained self-supervised learning framework with a shape-aware distance field. The start configuration $\cstart$ and goal configuration $\cgoal$ are utilized to predict the time $T(\cstart, \cgoal)$ through the time field regressor.
    The travel times $T(\cstart, \cgoal)$, $T(\cstart, \mathbf{w_i})$, and $T(\mathbf{w_i}, \cgoal)$ are employed to incorporate physical constraints during the training of the time field in a self-supervised manner to reduce local minima. 
    It is essential that the waypoint $\mathbf{w_i}$ remains collision-free, which can be ensured by distance $D(\mathbf{w_i})$ predicted through SADF.
    Moreover, $D(\mathbf{w_i})$ can also be converted into the ground truth speed $S^*(\mathbf{w_i})$ of $\mathbf{w_i}$ to compute the speed loss with the predicted speed $S(\mathbf{w_i})$, which is determined using the time $T(\cstart, \mathbf{w_i})$ and \cref{eik_eq}.
    During training, the motion planning (MP) iterates to derive the waypoint for physical constraint loss, while during testing, it iteratively computes waypoints to generate a path solution.
    }
    \label{fig:pipeline}
    \Description{Pipeline.}
\end{figure*}

\section{Preliminaries}\label{sec:NTFields}
\subsection{Robot Motion Planning}
Let $\mathcal{C}\subset \mathbb{R}^d$ and  $\mathcal{X} \subset \mathbb{R}^m$ represent the configuration space (c-space) of the robot and its surrounding environment, where $d, m\in \mathbb{N}$ denote the dimensions.
The obstructed c-space, formed by the obstacles in the environment $\mathcal{X}_{obs} \subset \mathcal{X}$, is denoted as $\mathcal{C}_{obs}$.
Additionally, the feasible space in c-space and the environment is represented as $\mathcal{C}_{free} = \mathcal{C}\backslash \mathcal{C}_{obs}$ and $\mathcal{X}_{free} = \mathcal{X}\backslash \mathcal{X}_{obs}$ respectively.
Given the robot start $\cstart \in \mathcal{C}_{free}$ and goal $\cgoal \in \mathcal{C}_{free}$ configurations, the objective of robot motion planning algorithms is to find a collision-free trajectory $\sigma = \{\mathbf{c_0}, ...,  \mathbf{c_i}, ..., \mathbf{c_n}\} \subset \mathcal{C}_{free}\  s.t.\  \mathbf{c_i} \in \mathcal{C}_{free}$, where $\mathbf{c_0} = \cstart$, $\mathbf{c_n} = \cgoal$, and $\mathbf{c_i}$ denotes the waypoint along the trajectory.

\subsection{NTFields}
The recent work NTFields \cite{ni2023ntfields} introduces a new perspective that relates motion planning problems with the solution to the Eikonal equation. The Eikonal equation defines the wave propagation with a first-order, nonlinear PDE formulation:
\begin{equation}
    S(\cgoal)^{-1} = \left \|\nabla_{\cgoal} T(\cstart,\cgoal)\right \|, \label{eik_eq}
\end{equation}
where $\cstart$ and $\cgoal$ are start and goal configurations, $T(\cstart,\cgoal)$ represents the arrival (travel) time from $\cstart$ to $\cgoal$, and $\nabla_\cgoal T(\cstart,\cgoal)$ denotes the partial derivative of arrival time with respect to the goal point. The solution to the equation yields the minimum arrival time of the wave from the start point to the goal point under the predefined speed model.

NTFields employs the multilayer perceptron (MLP) to model the time field $T$, which is the solution to the Eikonal equation. The neural network, parameterized by $\Theta$, takes the robot's start and goal configuration $(\cstart, \cgoal)$ as input and outputs the time field, namely, the Time Field Regressor:
\begin{equation}
    T_{\Theta}(\cstart,\cgoal) =\text{MLP}(\cstart,\cgoal;\Theta).
\end{equation}

The ground truth speed model is defined as:
\begin{equation}
S^{\ast }(\mathbf{c})=\frac{ S_{const}}{d_{max}}\times clip\left(D_{[r, \mathcal{X}_{obs}]}\left(\mathbf{c} \right),d_{min},d_{max} \right)\label{speedmodel},
\end{equation}
where $r$ denotes the robot and $D$ is a function that computes the shortest distance between the robot at configuration $\mathbf{c}\in \mathcal{C}$ and spatial obstacles $\mathcal{X}_{obs}$. This distance is obtained by sampling points on the surface of the robot and calculating the minimum distance from these points to the obstacles (implemented using BVH~\cite{Karras2012bvh}). The $clip$ function truncates the distance within the range of the minimum distance, $d_{min}$, and the maximum distance, $d_{max}$.
$S_{const}$ is a speed hyper-parameter, signifying that if the distance between the robot surface and the obstacle surpasses $d_{max}$, a maximum speed upper limit is imposed.
By providing the GT speed model which implicitly encodes obstacle information, NTFields can be trained with these GT speed and predicted speed which is derived from the predicted time according to \cref{eik_eq}.
However, the BVH-based speed model is computationally expensive, especially when the robot has a complex shape. 
Our SADF (detailed in \cref{sec:sadf}) offers an accelerated approach to calculating the speed field.

\subsection{Viscosity Solution of Eikonal Equation}

The non-uniqueness of the solution for the Eikonal equation can cause local minima when applied to motion planning as mentioned in \cite{sethian1999fmm}.
To solve this problem, the viscosity solution of the Eikonal equation is normally utilized with the following definition:
\begin{definition}
    \textit{The function $T$ is said to be the viscosity solution of the Eikonal equation provided that for all smooth test function $v$},

    \textit{1. if $T - v$ has a local maximum at a $\boldsymbol{x_0}$, then $G(\boldsymbol{x_0}, \nabla v(\boldsymbol{x_0})) \le 0$}

    \textit{2. if $T - v$ has a local minimum at a $\boldsymbol{x_0}$, then $G(\boldsymbol{x_0}, \nabla v(\boldsymbol{x_0})) \ge 0$}
\end{definition}
\noindent where $G(\mathbf{x}, \nabla T)$ is the general static Hamilton-Jacobi equation for the  Eikonal equation:
\begin{equation}
    G(\mathbf{x}, \nabla T) = 0.
\end{equation}

As proved in \cite{sethian1999fmm, crandall1983viscosity}, we have the following theorem:

\begin{theorem}[Existence and uniqueness]
    \label{thm:unique}
    There exists a unique viscosity solution $T$ of the Eikonal equation.
\end{theorem}

In other words, the uniqueness of the viscosity solution can inherently address the issue of local minima for the Eikonal equation. To construct the viscosity solution, Fast Marching Methods (FMM)~\cite{sethian1996fast} are commonly employed. FMM discretizes the c-space and ensures the viscosity solution by:

\begin{theorem}
    \label{thm:fmm_visc}
    As the discrete space size goes to zero, the numerical solution built by the Fast Marching Methods converges to the viscosity solution of the Eikonal equation.
\end{theorem}

The proof of this theorem is provided in \cite{barles1991convergence,tugurlan2008fast}. However, FMM has two limitations: 1) the discrete space size cannot practically reach zero implying that the numerical solution may not always converge to the viscosity solution; and 2) the computational burden of FMM increases significantly in high-dimensional spaces, as the number of discrete grids required grows dramatically with the dimensionality.

To overcome those limitations, we 
utilize neural networks to solve the Eikonal equation, which is more efficient. Meanwhile, instead of pursuing the exact viscosity solution, we introduce our physical constraints to reduce the local minima in the training process, which will be detailed in \cref{sec:phys_constraint}.

\section{Method}

\makeatletter
\renewcommand{\maketag@@@}[1]{\hbox{\m@th\normalsize\normalfont#1}}%
\makeatother

For any start-goal pair, our \pcplanner employs the time field to compute the travel time and its partial derivatives, which are then used in an iterative motion planning process to generate the final path. To train a time field for a new environment, we propose a physics-constrained self-supervised training framework (\cref{sec:phys_constraint}) incorporating a SADF (\cref{sec:sadf}) that is crucial in data preprocessing (GT generation), training (efficient collision-checking), and testing (rapid collision-checking in adaptive motion planning (\cref{sec:adapt_mp})).

As illustrated in ~\cref{fig:pipeline}, in our self-supervised framework, pairs of the start and goal configurations are first regressed into the time field through \textit{Time Field Regressor}. 
Subsequently, the time field is employed to predict the speed and determine the waypoint along the trajectory through motion planning (\textit{MP} module). The physical constraints are then applied to the start, waypoint, and goal configurations in a self-supervised manner, 
which can help the network escape local minima and converge to the correct solutions consistent with physical principles.
Moreover, the proposed \textit{SADF Predictor} is utilized to obtain the shortest distance of any given configuration to the environment.
It facilitates collision checking to ensure the collision-free status of the start, waypoint, and goal configurations for the physical constraints and aids in generating GT speed fields in the data processing stage.
For ease of illustration, the visualized time field represents the travel time from a fixed start configuration $\mathbf{c_s}$ to any goal configuration $\mathbf{c_i} \in \mathcal{C}$ for the robot and the visualized speed field denotes the speed of the robot at any configuration $\mathbf{c} \in \mathcal{C}$. Additionally, both fields are visualized in 2D space for clarity.

\subsection{Physics-Constrained Self-Supervised Learning}
\label{sec:phys_constraint}

As previously mentioned, the local minima in the solution of the NTFields are caused by the non-uniqueness of the solution for the Eikonal equation.
Motivated by traditional methods like FMM, in this section, we introduce a novel self-supervised strategy with two physical constraints related to the viscosity solution, \emph{monotonic constraint} and \emph{optimal constraint}, to enhance the physics-informed neural motion planner by escaping local minima.

\noindent\textbf{Monotonic Constraint.} From Theorems~\ref{thm:unique} and \ref{thm:fmm_visc}, we have the following property \cite{sethian1999fmm}:

\begin{property}
    \label{prop:monotone}
    Travel time solved by FMM increases monotonically away from the start point towards the goal.
\end{property}
\noindent This indicates that the travel time through any waypoint in the path does not shortcut the overall travel time from start to goal. 
Based on this, we introduce the monotonic constraint which implicitly enforces the monotonicity property in the neural network.

Specifically, for any waypoint along the planned path, the total travel time from the start to the goal configuration must exceed both the travel time from the start to the waypoint and from the waypoint to the goal configuration. More formally, let $T(\mathbf{x}, \mathbf{y})$ denote the travel time from configuration $\mathbf{x}$ to $\mathbf{y}$. We have:
\begin{align}
    \label{T_sw}
    &T(\mathbf{s}, \mathbf{g}) > T(\mathbf{s}, \mathbf{w}), \\
    \label{T_wg}
    &T(\mathbf{s}, \mathbf{g}) > T(\mathbf{w}, \mathbf{g}),
\end{align}
where $\mathbf{s}, \ \mathbf{g} \ \text{and}\ \mathbf{w} \in \mathcal{C}_{free}$ represent start, goal, and waypoint configuration respectively in the free space. 
The significance of the monotonic constraint in maintaining the fidelity of the viscosity solution can be verified by the following proposition:

\begin{proposition}
    If the solution $T$ violates the monotonic constraints, then $T$ is not the viscosity solution.
\end{proposition}

\begin{proof}
    If $T$ violates the monotonic constraints, it fails to satisfy the property described in \Cref{prop:monotone}, which is necessary for being a solution generated by FMM. The limit of the FMM solutions also adheres to \Cref{prop:monotone}, thereby indicating that $T$ cannot be the limit of FMM solutions. Since the solution derived from FMM converges to the viscosity solution via \Cref{thm:unique}, and considering the uniqueness of the viscosity solution according to \Cref{thm:unique}, it follows that $T$ is not the viscosity solution.
\end{proof}

In this way, our monotonic constraint implicitly forces the network to converge to the viscosity solution.
We also observe that the occurrence of local minima in the time fields, depicted in \cref{fig:time_field}, is associated with waypoints that violate the specified monotonic constraints mentioned above.
Therefore, we can exploit this constraint to guide the network in a self-supervised learning process, allowing it to correct and attain an accurate time field.

However, since our objective is to find the optimal path through the Eikonal equation, obtaining a ground-truth waypoint on the optimal path before receiving the correct solution from the Eikonal equation proves impractical. A straightforward strategy involves employing a traditional path planning approach, such as FMM, to find the waypoint and thus guide the learning. 
However, this comes at a significant cost due to the slow and cumbersome nature of traditional methods.

To address this challenge, we integrate a self-correction mechanism in a self-supervised manner into network learning. During training, we utilize the network to plan a path and designate any segment of the path as our waypoint.
Our goal is to prompt the network to recognize that the current generated path violates the monotonic constraint.
This is achieved by integrating the monotonic constraint loss into the network:
\begin{equation}
    \begin{aligned}
        \mathcal{L}_{m} = 
        &\  \sum \max \left(  T(\mathbf{s}, \mathbf{w}) - T(\mathbf{s}, \mathbf{g}), 0 \right) + \\
        &\  \sum \max \left(  T(\mathbf{w}, \mathbf{g}) - T(\mathbf{s}, \mathbf{g}), 0 \right).
    \end{aligned}
\end{equation}
While, from a global perspective, the chosen waypoint in each training batch does not represent the waypoint of the eventual optimal path, the network perceives it as optimal during the training step. Meanwhile, the selection of the waypoint evolves with the monotonic constraint loss and will be optimized during each training batch of the network. This self-supervised training approach with the monotonic constraint loss demonstrates improvement in addressing local minima, as illustrated in \cref{fig:time_field}.

\noindent\textbf{Optimal Constraint.}
To further refine the motion planning process, we introduce the \emph{optimal constraint}, which penalizes the path for deviating from the optimality criterion of the Eikonal equation's solution:
\begin{equation}
    \label{T_sg}
    T(\mathbf{s}, \mathbf{g}) \leq T(\mathbf{s}, \mathbf{w}) + T(\mathbf{w}, \mathbf{g}).
\end{equation}
Similar to the previously discussed monotonic constraint loss, we introduce an optimal constraint loss term as follows:
\begin{equation}
    \mathcal{L}_{o} = \sum \max \left( T\left(\mathbf{s}, \mathbf{g}\right) - \left(T\left(\mathbf{s}, \mathbf{w}\right) + T\left(\mathbf{w}, \mathbf{g}\right) \right ), 0 \right).
\end{equation}

Following NTFields, we utilize the isotropic speed loss to govern the training of the time field:
\begin{equation}
    \label{L_speed}
    \fontsize{7.9}{8.9}\selectfont
        \begin{aligned}
        \mathcal{L}_{s} = \frac{1}{|\mathcal{C}|} \sum_{\cfst, \csec \in \mathcal{C}} \Big (
        &  \left \| 1-\sqrt{S^{\ast }(\cfst ) /S_{\Theta}(\cfst)}  \right \| + \left \| 1-\sqrt{S^{\ast }(\csec ) /S_{\Theta}(\csec)}  \right \| + \\
        & \left \| 1-\sqrt{S_{\Theta}(\cfst) /S^{\ast }(\cfst )}  \right \|+  \left \| 1-\sqrt{S_{\Theta}(\csec) /S^{\ast }(\csec )}  \right \| \Big ),
        \end{aligned}
\end{equation}
where $|\mathcal{C}|$ denotes the number of sampled configurations.
$\cfst$ and $\csec$ represent an arbitrary pair of the start and goal configurations in the c-space. $S^{\ast }(\cfst)$ and $S^{\ast }(\csec)$ are the GT speed values calculated by our SADF (introduced in \Cref{sec:sadf}) according to the speed model in \cref{speedmodel}. $S_{\Theta}(\cfst)$ and $S_{\Theta}(\csec)$ are the predicted speed values derived from the predicted time field according to \cref{eik_eq}.

With the guidance of our physical constraints, the total loss function for the training of the time field is defined as:
\begin{equation}
    \mathcal{L} = \mathcal{L}_{s} + \lambda_m\mathcal{L}_{m} + \lambda_o\mathcal{L}_{o},
\end{equation}
where $\lambda_m$ and $\lambda_o$ control the penalty weights for the monotonic constraint and optimal constraint respectively.

\noindent\textbf{Training details.}
To calculate the speed loss during training, we follow these steps:
1) sample configuration pairs $[\cfst, \csec]$ in the robot's c-space;
2) derive the GT speed values $[S^{\ast }(\cfst), S^{\ast }(\csec)]$ calculated by \cref{speedmodel};
3) approximate $T_{\Theta}(\cfst, \csec)$ via time field regressor;
4) obtain the predicted speed values $[S_{\Theta}(\cfst), S_{\Theta}(\csec)]$ according to \cref{eik_eq};
Finally, the speed loss \cref{L_speed} can be derived via the GT speed and predicted speed.
For the physical constraint loss, we adopt a different paradigm:
1) sample configuration pairs $[\cfst, \csec]$ in the feasible c-space $\mathcal{C}_{free}$;
2) leverage the time field regressor to obtain $T(\cfst, \csec)$;
3) utilize the gradient of the time field to conduct motion planning, which determines the next waypoint $\mathbf{c_w}$. If the waypoint is in collision, jump to step 1 to resample a new configuration pair;
4) compute $T(\cfst, \mathbf{c_w})$ and $T(\mathbf{c_w}, \csec)$ through time field regressor.
Ultimately, the physical constraint loss can be calculated with $T(\cfst, \csec)$, $T(\cfst, \mathbf{c_w})$ and $T(\mathbf{c_w}, \csec)$.

\subsection{Shape-Aware Distance Function}
\label{sec:sadf}

When a shaped robot is navigating in the environment, it becomes insufficient only to obtain the distance field of the environment, as a real-world robot cannot be simplistically treated as a particle but rather as a rigid or even deformable body. 
Therefore, we propose the Shape-Aware Distance Function (SADF) to address the requirements of real-shaped robots.
We first introduce how to derive the SADF for the rigid robots and then extend it to the articulated robots.

\noindent\textbf{Rigid Robot.}
Following the definition of Signed Distance Function (SDF), 
the Shape-Aware Distance Field (SADF) is defined as the distance between the surface of the robot and the environment:
\begin{align}
    \label{eq:sadf_def}
    &f_{[r, e]}(H) = \min_{\mathbf{x} \in r, \mathbf{y} \in e}d(H(\mathbf{x}), \mathbf{y}),\\
    &\text{with}\quad H(\mathbf{x}) = R\mathbf{x} + \mathbf{t}, \quad R \in \mathbb{RO}(3),\ \mathbf{x}, \mathbf{t} \in \mathbb{R}^3
\end{align}
where $r$ and $e$ denote the surface boundaries of the robot and the environment, $H \in \mathbb{SE}(3)$ denotes the relative transformation from the robot's coordinate system to the environment's coordinate system, and $d$ is the distance function between any two points.

\begin{figure}[!t]
    \centering
    \includegraphics[width=\linewidth]{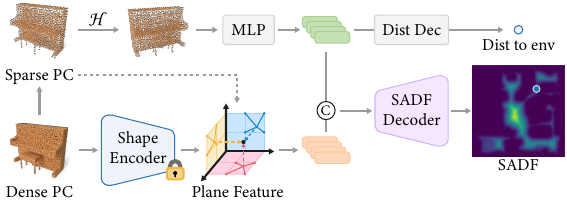}
    \caption{Training pipeline of SADF. During inference, "Dist Dec" is omitted, with only "SADF Decoder" branch employed as a SADF predictor.}
    \label{fig:sadf}
    \Description{The architecture of the proposed Shape-Aware Distance Field.}
\end{figure}

Given a raw point cloud $P = \{p_i \in \mathbb{R}^3\}_{i=1}^{N}$ comprising $N$ points from the robot's surface and considering the transformations $H$ of the robot, our objective is to learn the SADF between an arbitrary robot and a fixed environment, denoted as $\Phi_e(P, H) = f_{[P, e]}(H)$.

To obtain a precise SADF, it is necessary to sample dense points on the robot. However, directly calculating the distance from every point to the environment and obtaining the minimum for every transformation $H$, following the definition (\cref{eq:sadf_def}), can be computationally burdensome 
and inefficient.
To tackle the problem, we propose employing a sparse point cloud and local shape feature derived from the dense point cloud to characterize the robot, and use the neural network to approximate the SADF.

Specifically, the sparse point cloud $P_{s}$, downsampled from the dense point cloud $P_d$, first undergoes transformation based on $H$ and is then processed into the distance feature $\mathbf{F}_d$
through a small MLP parameterized by $\Theta_{1}$. The distance feature, which encapsulates the transformation information between the points and the environment, is then decoded to derive the distance from a point to the environment via a point distance decoder $\Theta_{2}$.
The loss function to govern the MLP and decoder is defined as:
\begin{equation}
    \fontsize{8.5}{9.5}\selectfont
    \mathcal{L}_d = \frac{1}{|\mathcal{H}||P_{s}|} \sum_{H \in \mathcal{H}, \mathbf{x} \in P_{s}} \| \hat{\operatorname{SDF}}_e(H(\mathbf{x});\Theta_{1}, \Theta_{2}) - \operatorname{SDF}_e(H(\mathbf{x})) \|, 
\end{equation}
where $|P_{s}|$ and $|\mathcal{H}|$ denote the number of sparse point cloud and the number of sampled transformations respectively, $\operatorname{SDF}_e(\cdot)$ represents the GT SDF of the environment, and $\hat{\operatorname{SDF}}_e(\cdot;\Theta_{1}, \Theta_{2})$ indicates the learned SDF function parameterized by $\Theta_{1},\Theta_{2}$.

Meanwhile, we utilize the pre-trained encoder from \cite{chou2022gensdf} to extract the plane features of the robot from densely sampled points $P_d$. Then, we aggregate the local shape feature $\mathbf{F}_s$ for each sparse point via bilinear interpolation in the plane feature space. 
At last, we concatenate $\mathbf{F}_d$ and $\mathbf{F}_s$ of each sparse point and feed them to the SADF decoder $\Theta_{3}$ to decode the ultimate SADF in the environment, as shown in \cref{fig:sadf}.
The loss function for the SADF training is as follows: 
\begin{equation}
    \fontsize{7.8}{8.8}\selectfont
    \mathcal{L}_{SADF} = \left(\frac{1}{|\mathcal{H}|}\sum_{H \in \mathcal{H}} \|\hat{\Phi}_{e}(P_d, H;\Theta_{1},\Theta_{3},\Theta_{4}) - \Phi_{e}(P_d, H) \|\right) + \lambda_d \mathcal{L}_d,
\end{equation}
where $\Theta_{4}$ denotes the frozen parameters of the shape encoder, $\hat{\Phi}_{e}(\cdot, \cdot;\Theta_{1},\Theta_{3},\Theta_{4})$ represents the learned SADF parameterized by $\Theta_{1},\Theta_{3},\Theta_{4}$, and $\lambda_d$ determines the weights of $\mathcal{L}_d$.
Note that the encoded shape feature $\mathbf{F}_s$, derived from the dense point cloud, only needs to be computed once for a specific robot. This will significantly reduce the overall computational cost.

\noindent\textbf{Articulated Robot.}
To model the articulated robots, we draw inspiration from \cite{li2024representing} and represent them by the union of each part's SADF.
Consider the articulated robot with $m$ links, characterized by shapes $\boldsymbol{l}=\{l_0, l_1, ..., l_{m-1}\}$, the SADF for the articulated robot can be represented by the union of each link's SADF, which is defined as:
\begin{equation}
    \begin{aligned}
        f_{[r, e]}(H_b) = \min \{f_{[l_i, e]}(H^b_i H_b)\}, \ i = \{0, 1, ..., m-1\},
    \end{aligned}
\end{equation}
where $H_b$ denotes the transformation from the base frame of the robot to the world frame, and $H^b_i$ represents the transformation from the $i$-th link's frame to the base frame.

\subsection{Adaptive Motion Planning}
\label{sec:adapt_mp}

\begin{figure*}[!t]
    \centering
    \includegraphics[width=\linewidth]{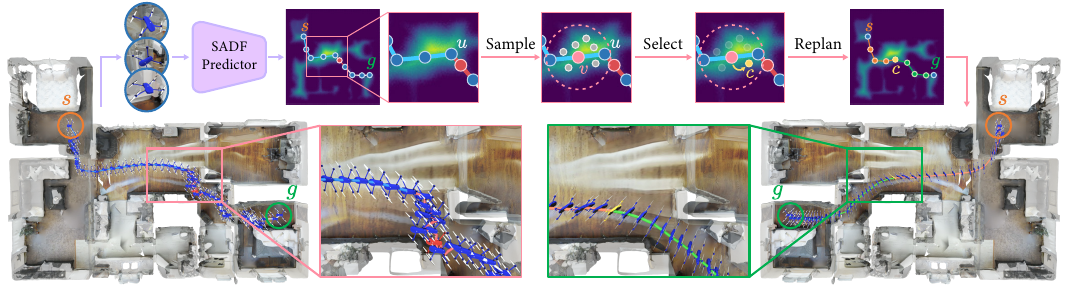}
    \caption{
        Adaptive motion planning.
        We locate a collision-free waypoint $\mathbf{u}$ adjacent to the collision point and randomly sample points around $\mathbf{u}$. These sampled points serve as candidates to escape local minima. The candidate point $\mathbf{c}$ is selected based on traversal times calculated from $\cstart$ to $\mathbf{c}$ and from $\mathbf{c}$ to $\cgoal$. 
    }
    \label{fig:adaptive_mp}
    \Description{Adaptie motion planning}
\end{figure*}

Due to the unpredictability of the network and the fact that our physical constraints serve as a necessary condition for the viscosity solution of the Eikonal equation but not a sufficient one, it remains challenging to avoid all of the local minima entirely.
Therefore, we introduce an adaptive motion planning approach illustrated in~\cref{fig:adaptive_mp} to bolster the robustness of the planning procedure, leveraging the fast inference speed of our SADF for collision checking.

For the collision-free case, the ultimate path solution is obtained by performing gradient descent bidirectionally, traversing from the start to the goal and vice versa, which follows the trivial planning strategy of NTFields:
\begin{gather}
    \cstart_{i+1}  = \cstart_{i} +\alpha S^{2}(\cstart_i)\nabla _{\cstart_i} T(\cstart_i,\cgoal_i)\label{caulate_path_start}, \\
    \cgoal_{i+1}  = \cgoal_{i} +\alpha S^{2}(\cgoal_i)\nabla _{\cgoal_i} T(\cstart_i,\cgoal_i)\label{caulate_path_goal},
\end{gather}
where $\alpha$ is a step size hyperparameter, and $S^{2}(\cstart_i)$ or $S^{2}(\cgoal_i)$ is a regulation coefficient to address safety issues arising from low speeds near obstacles, which can cause large gradients due to the inverse speed-gradient relationship in \cref{eik_eq}.

If a collision is detected using our SADF, we employ adaptive motion planning. As shown in~\cref{fig:adaptive_mp},  we proceed to find a collision-free waypoint $\mathbf{u}$ adjacent to the collision point,
thereby forming the longest collision-free path from $\mathbf{u}$  to either the start or goal configuration. 
Then, we randomly select a point $\mathbf{v}$ along such a collision-free path for replanning.
Random sampling employed here is to increase the probability that subsequently sampled points are also collision-free.
Similar to the initialization of the sampling-based method, we randomly sample an appropriate number of points within a hypersphere in the configuration space, with $\mathbf{v}$ as the center and $r$ as the radius. These points serve as candidate points to escape local minima.
Subsequently, the candidate point $\mathbf{c}$ can be utilized as a waypoint on the path to devise a new collision-free trajectory which is formed as $\cstart-\mathbf{c}-\cgoal$.
To retain the optimality, we calculate the sum of $T(\cstart,\mathbf{c})$ and $T(\mathbf{c},\cgoal)$ to obtain the total time of the replanned path from the learned time field for each candidate point $c$. We then choose the candidate point with the shortest time to form the final trajectory.
Moreover, we adaptively enlarge the search radius $r$ if a collision-free path is not found within the candidate points.
This adaptive strategy aids in exploring a larger configuration space to find feasible paths when necessary, enhancing the robustness of our \pcplanner.
It's worth noting that this strategy is cost-affordable as it is applied only on the collision path (which is rare in practice thanks to our physics-constrained learning).
\section{Experiments}

In this section, we analyze and validate our method with different robots and environments.
We compare our methods with the baselines NTFields~\cite{ni2023ntfields}, P-NTFields~\cite{ni2023progressive}, FMM~\cite{sethian1996fast}, RRT*~\cite{kingston2018sampling}, RRT-Connect~\cite{kuffner2000rrt}, and LazyPRM*~\cite{bohlin2000path}.
For RRT*, RRT-Connect, and LazyPRM* which are probabilistically complete and can theoretically find a solution given infinite time, we impose a practical time limit~(10 seconds for rigid robots and 5 seconds for manipulators), since real-world scenarios often require time-constrained solutions.
Our evaluation metrics include \textbf{path length}, \textbf{planning time}, \textbf{success rate~(SR)}, and \textbf{challenging success rate (CSR)}.
For additional information on the experimental settings, including metrics description, baseline details, and experimental specifics, please refer to our supplementary material.

\subsection{Motion Planning in 3D Environments for Rigid Robots}

\begin{table}[!t]
\centering
\caption{
    Comparison of motion planning in 3D for rigid robots. The optimal results are highlighted with \colorbox{colorFst}{\bf first}, \colorbox{colorSnd}{second}.
}

\footnotesize
\setlength{\tabcolsep}{1.5pt}

\resizebox{\linewidth}{!}{
\scalebox{0.85}{

\begin{tabular}{lccccccc}
\toprule[0.15em]
\multirow{2}{*}{Methods}  & \multirow{2}{*}{Metrics} & \multicolumn{3}{c}{Arona} & \multicolumn{3}{c}{Eastville} 
\\ \cmidrule(lr){3-5} \cmidrule(lr){6-8} 
&     & Bear      & Mobile Root  & Bird($\mathbb{SE}(3)$)    & Piano     & Toy Car & Drone($\mathbb{SE}(3)$)       \\
\midrule[\ourmidrulewidth]

\multirow{\metricsize}{*}{RRT*}      
& Length                   & 0.26          & 0.26        & 0.23       & 0.19          &\nd 0.23         & \nd 0.23      \\
& Time(ms)                 & 10189.2       & 10203.1     & 10121.2      & 10215.2       & 10226.5       & 10196.6 \\
& SR(\%)                   & 83.7          & 81.8        &  89.6      & 80.4          & 80.7            & 83.7  \\ 
& CSR(\%)                  & 84.4          & 37.3        &  29.3     & 60.5          & 63.0             & 29.2\\
\midrule
\multirow{\metricsize}{*}{LazyPRM*}  
& Length                   & \nd 0.25      & \fs 0.24     & \fs 0.21     & \nd 0.18      & \fs0.20     &\fs 0.19  \\
& Time(ms)                 & 10158.9       & 10152.6      & 10121.9  & 10152.3       & 10252.3    &  10130.5 \\
& SR(\%)                   & 84.6          & 87.4         & 92.9  & 81.6          & 87.3         &   90.3 \\
& CSR(\%)                  & 85.6          & 64.8         & 51.7  & 68.4          & 75.2         &  57.1 \\
\midrule
\multirow{\metricsize}{*}{RRT-Connect}
& Length                   & 0.70          & 0.72        & 1.1        & 0.66          & 0.70       & 1.03   \\
& Time(ms)                 & 1543.4        & 1246.2      & 959.3      & 2008.1        & 1568.7     & 1259.1    \\
& SR(\%)                   & 90.0          & 91.5        &  \nd 93.6  & 85.5          & 89.4       & 89.7   \\
& CSR(\%)                  & 94.6          & \nd 77.9        &  \nd60.5   & 72.8          & 81.0       &  \nd 68.6   \\
\midrule[\ourmidrulewidth]
\multirow{\metricsize}{*}{NTFields}  
& Length                   & \nd 0.25      & \nd 0.25        & \fs 0.21     & \nd 0.18      & \fs0.20     & \fs 0.19   \\
& Time(ms)                 & 6.1           & 6.1             & \nd 4.5      & \nd 5.7       & 6.7         & \nd 2.5   \\
& SR(\%)                   & 85.4          & 77.3            & 88.5         & 86.1          & 94.3        & 86.4  \\
& CSR(\%)                  & 56.3          & 45.6            & 22.5         & 72.4          & 88.7        & 42.0 \\
\midrule
\multirow{\metricsize}{*}{P-NTFields}    
& Length                   & \fs 0.24      & \fs 0.24       & \fs 0.21        & 0.2           & \nd0.21      & \fs 0.19     \\
& Time(ms)                 & \nd 2.8       & \nd 2.7        & \fs 1.4         & 6.5           & \nd4.2       & \fs 1.7  \\
& SR(\%)                   & 94.9          &  \nd 96.7         & 86.3         & 80.4          & 88.9         & 85.4   \\   
& CSR(\%)                  & 85.0          & \fs 91.1       & 7.5             & 61.5          & 79.8         & 35.4   \\
\midrule[\ourmidrulewidth]
\multirow{\metricsize}{*}{\begin{tabular}[c]{@{}l@{}}Ours w/o\\ adapt. planning\end{tabular}}
& Length                   & \fs 0.24          & \fs 0.24     & \fs 0.21      & \fs 0.17     & \fs0.20         & \fs 0.19\\
& Time(ms)                 & \fs 1.9           & \fs 2.3      & 4.9           & \fs 1.9      & \fs1.7          & 4.6  \\
& SR(\%)                   & \nd 99.7          & 96.0         & \fs 95.8      & \nd91.3      & \nd 96.2     & \nd 92.7     \\ 
& CSR(\%)                  & \nd 99.1          & 88.8         & \fs 72.1      & \nd83.0      & \nd 92.5        &   67.7     \\
\midrule
\multirow{\metricsize}{*}{Ours}   
& Length                  & \fs 0.24          & \fs 0.24       & \fs 0.21   & \fs 0.17     & \fs0.20           & \fs 0.19  \\
& Time(ms)                & \nd 2.8           & 26.0           & 4.6        & 49.0         & 25.1              & 38.5\\ 
& SR(\%)                  & \fs 99.8          & \fs 96.8       & \fs 95.8   & \fs 92.6     & \fs 96.9         & \fs 93.2      \\ 
& CSR(\%)                 & \fs 99.4          & \fs 91.1       & \fs 72.1   & \fs 85.6     & \fs 93.5         & \fs  70.0       \\

\toprule[0.15em]
\end{tabular}
 }
}

\label{table:3d_se2_se3}

\end{table}

We perform a comparative analysis of our method against the baselines in two complex 3D Gibson environments~(Arona and Eastville)~\cite{xiazamirhe2018gibsonenv} in $\mathbb{SE}(2)$ and $\mathbb{SE}(3)$ space with rigid robots. 
In those two Gibson environments, we employ different robots. In one scenario, we deploy a mobile robot, a bear, and a bird for navigation within the environment. In the other scenario, we showcase the use of a piano, a toy car, and a drone to plan in the environment. Specifically, the mobile robot, bear, piano, and toy car navigate in $\mathbb{SE}(2)$ space with 3 Degrees of Freedom (DoFs) while the bird and drone plan in $\mathbb{SE}(3)$ space with 6 DoFs.
These planning examples are depicted in \cref{fig:result}. 
Additional experiments on 2D environments in $\mathbb{SE}(2)$ space are demonstrated in the supplementary material.

The quantitative results are listed in \cref{table:3d_se2_se3}. In the Arona environment, our proposed method demonstrates the best SR, CSR, and path length with a competitive computation time. It is worth mentioning that without adaptive planning, our method achieves minimal computation time and best or second-best SR, CSR and path length in most tasks. Moreover, our method is more efficient compared to traditional methods, ranging from at least 40 times to as much as 200 times faster than traditional methods, indicating an affordable time cost of our adaptive strategy.
This highlights the effectiveness and superiority of our proposed method.

Additionally, we validate our methods on the real-world \textbf{TurtleBot4} robot in the complex meeting room and hallway environments with clustered obstacles, as shown in \cref{fig:real_robot}.

\begin{figure*}
    \centering
    \includegraphics[width=\linewidth]
    {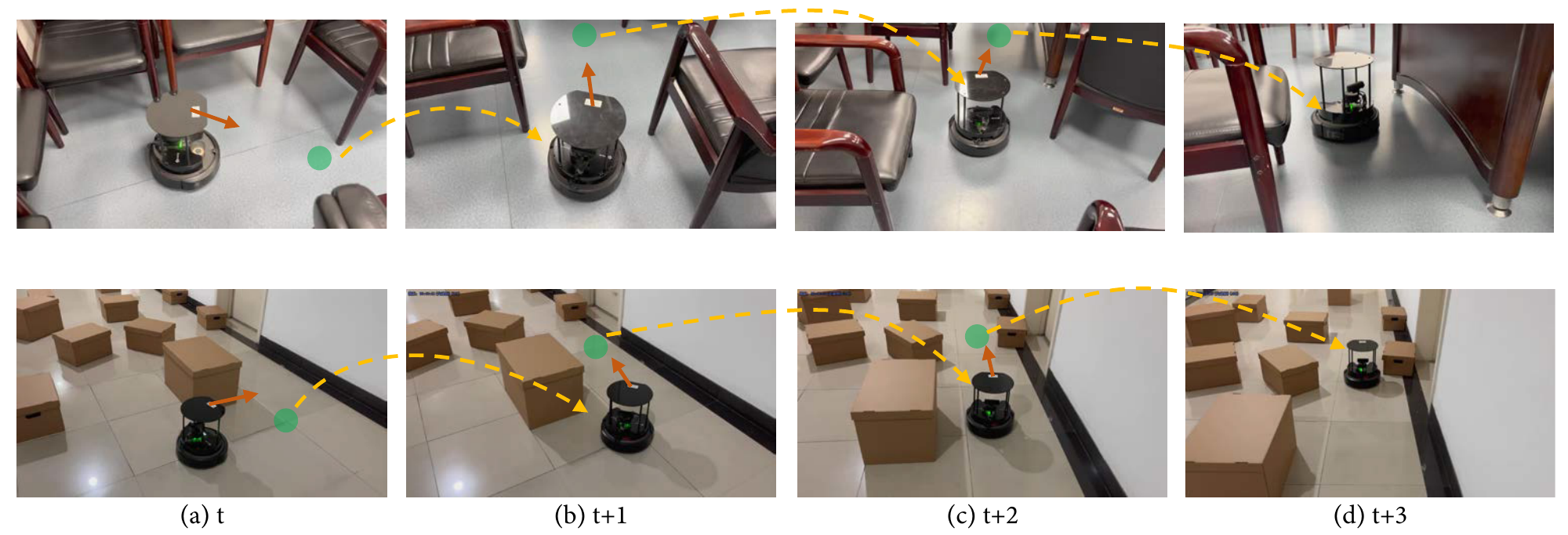}
    \caption{A real TurtleBot4 robot navigates in the real-world meeting room with clustered chairs and the hallway with cluster boxes. }
    \label{fig:real_robot}
    \Description{Real robot results.}
\end{figure*}

\begin{table}[!t]
\centering
\caption{
    Comparison of motion planning for manipulators.
    The optimal results are highlighted with \colorbox{colorFst}{\bf first}, \colorbox{colorSnd}{second}.
}
\footnotesize
\setlength{\tabcolsep}{5pt}

\begin{tabular}{lccccc}
\toprule[0.15em]
\multirow{3}{*}{Methods}
& \multirow{3}{*}{Metrics} & \multicolumn{2}{c}{Manual Craft} & Single-cabinet  & Dual-cabinet  \\
&  & \multicolumn{2}{c}{custom-arm} &  UR5 arm & UR5 arm \\
\cmidrule(lr){3-4} \cmidrule(lr){5-5} \cmidrule{6-6}
&                          & 4-DoF             & 6-DoF
&                          6-DoF              & 6-DoF                       \\

\midrule[\ourmidrulewidth]
\multirow{\metricsize}{*}{RRT*}
& Length                   & \nd 0.28         &  0.23                &  0.35              & \nd 0.25                 \\
& Time(ms)                 & 5120             & 5120                 & 5140               & 5130    \\
& SR(\%)                   & 90.1             & 90.5                 & 88.1               & 85.6       \\
& CSR(\%)                  & 55.5             & 48.4                 & 40.5               & 41.0 
  \\
\midrule
\multirow{\metricsize}{*}{LazyPRM*}
& Length                   &  \fs 0.25        &  \nd 0.21            & \fs 0.28          & \fs 0.21                \\
& Time(ms)                 & 5080             & 5060                 & 5070               & 5080       \\
& SR(\%)                   & 98.2             & 97.9                 & 98.8               & 97.5          \\
& CSR(\%)                  & 86.4             & 87.9                 & 85.1               & 85.6        \\
\midrule
\multirow{\metricsize}{*}{RRT-Connect}
& Length                   & 0.83             & 1.04                & 1.04               & 1.03            \\
& Time(ms)                 & 370              & 460                 & 360                & 810           \\
& SR(\%)                   & 97.4             & 97.4                & \nd 98.9           & 96.7           \\
& CSR(\%)                  &  93.6            & 93.6                & 89.2               & \fs 89.0          \\
\midrule[\ourmidrulewidth]
\multirow{\metricsize}{*}{NTFields}
& Length                   & \fs 0.25         & \fs 0.20            & \fs 0.28           & \fs 0.21           \\
& Time(ms)                 & 4.9              & \nd 1.5             & \nd 1.7            & \nd 1.9             \\
& SR(\%)                   & 91.2             & 96.0                & 97.1               & 93.1         \\
& CSR(\%)                  & 60.0             & 74.5                & 60.8               & 60.7           \\
\midrule
\multirow{\metricsize}{*}{P-NTFields}
& Length                   & \fs 0.25         & \fs 0.20             & \fs 0.28           & \fs 0.21        \\
& Time(ms)                 & \fs 1.0          & \fs 1.1              & \fs 1.2            & \fs 1.3        \\
& SR(\%)                   & 89.7             & 87.9                 & 95.5               & 84.9        \\
& CSR(\%)                  & 53.2             & 24.2                 & 40.5               & 14.5        \\
\midrule[\ourmidrulewidth]
\multirow{\metricsize}{*}{\begin{tabular}[c]{@{}l@{}}Ours w/o\\ adapt. planning\end{tabular}}
& Length                   & \fs 0.25         & \fs 0.20             & \fs 0.28              & \fs 0.21      \\
& Time(ms)                 & \nd 1.2          & \fs 1.1              & \fs 1.2               & 2.9         \\
& SR(\%)                   & \nd 99.0         & \nd 99.4             & \fs99.6               & \nd 97.8      \\
& CSR(\%)                  & \nd 95.5         & \nd 96.2             & \fs 96.0              & 87.3      \\
\midrule
\multirow{\metricsize}{*}{Ours}
& Length                   & \fs 0.25         & \fs 0.20              & \fs 0.28              & \fs 0.21   \\
& Time(ms)                 & 2.5              & 2.0                   & 18.4                  & 58.3  \\
& SR(\%)                   & \fs 99.4         & \fs 99.8              & \fs99.6               & \fs 97.9  \\
& CSR(\%)                  & \fs 97.3         & \fs 98.7              & \fs96.0               & \nd 87.9 \\

\toprule[0.15em]
\end{tabular}

\label{table:arm_simple}

\end{table}

\begin{table}[!t]
\centering
\caption{
    Performance comparison of \pcplanner, specifically under conditions without adaptive motion planning, using our SADF against BVH-distance-query with different sampled surface points. PC indicates the physical constraints. The optimal results are highlighted with \colorbox{colorFst}{\bf first}, \colorbox{colorSnd}{second}.
}
\footnotesize
\begin{tabular}{lc|l|cccc}
\toprule[0.15em]
\multirow{2}{*}{Methods} & \multirow{2}{*}{Points} & \multirow{2}{*}{Metric} & \multicolumn{2}{c}{Aronna}  & \multicolumn{2}{c}{EastVille} \\ \cmidrule(lr){4-5} \cmidrule(lr){6-7}
&                         &              & Bear & Mobile Robot & Piano & Car \\ \midrule
NTFields                  &  1024 & \multirow{4}{*}{SR(\%)} & 85.4 & 77.3 & 86.1 & 94.3   \\
PC + BVH                &              1024 &    & \fs99.9 & \fs97.0 & \fs96.1 & \fs96.6 \\
PC + BVH                &                32 &    & 95.9 & 93.7 & 89.5 & 95.9  \\ 
PC + SADF               &                32 &        & \nd99.7 & \nd96.0 & \nd91.3 & \nd96.2 \\
                \bottomrule[0.15em]
\end{tabular}


\label{table:ablation_points}
\end{table}

\begin{figure*}
    \centering
    \includegraphics[width=\linewidth]{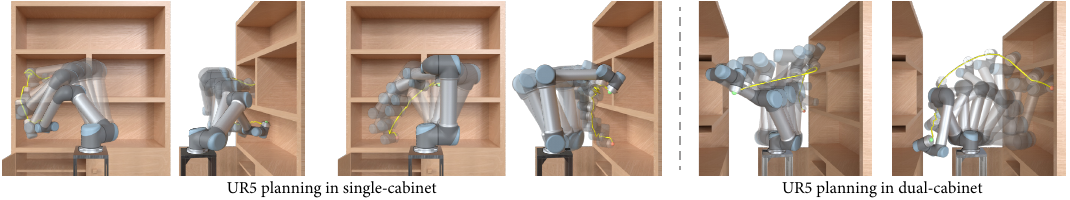}
    \caption{The UR5 manipulator executes motion planning in the single-cabinet and dual-cabinet environment.}
    \label{fig:ur5_plan}
    \Description{Ur5 planning result.}
\end{figure*}


\begin{figure}
    \centering
    \includegraphics[width=\linewidth]{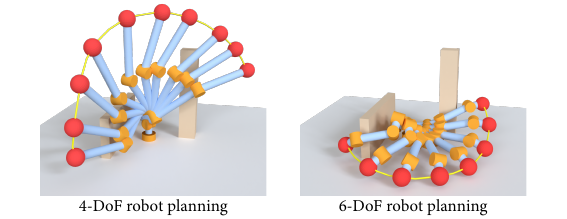}
    \caption{4-DoF and 6-DoF custom-built arms do motion planning in the manually crafted environment.}
    \label{fig:cutom_arm_plan}
    \Description{Custom Arm planning results.}
\end{figure}

\subsection{Motion Planning for Manipulators}
This section showcases the performance of our method on both custom-built and standard UR5 manipulators across various scenarios.
These scenarios include a simple manually crafted environment, a single-cabinet operating environment, and a dual-cabinet opposing operating environment, each with increasing levels of difficulty.

As shown in ~\cref{table:arm_simple}, we conducted 4-DoF and 6-DoF experiments with a custom-built arm in the manually crafted environment. Our method outperformed in all metrics except for the time. However, the time difference compared to the best performance is almost negligible.
The experiments with the 6-DoF UR5 manipulator are presented in the last two columns of ~\cref{table:arm_simple}. In the single-cabinet environment, our method continues to top all other metrics, even without adaptive planning. In the dual-cabinet opposing environment, ours maintains the highest SR and ranks second in CSR under challenging cases. Planning examples are depicted in \cref{fig:ur5_plan,fig:cutom_arm_plan}.

\begin{table}[!t]
\centering
\caption{
    Comparison of collision-checking time among FCL, BVH-distance-query, and our SADF with different sampled points on various numbers (10, $\dots$, 10000) of relative transformation $H$ between robot and environment.}
\footnotesize
\setlength{\tabcolsep}{7pt}
\begin{tabular}{lcccc}
\toprule[0.15em]

&\multicolumn{4}{c}{Time (ms)} \\
\cmidrule(lr){2-5}
Transformation Numbers & 10 & 100 & 1000 & 10000 \\
\midrule
FCL                                  & 1.0 & 4.9 & 40.4 & 392.6 \\ 
BVH  + 1024 points               & 5.4 & 9.1 & 39.0 & 382.1  \\
BVH  + 32 points                & 5.8 & 5.7 & 6.7 & 31.2   \\
SADF + 32 points           & \textbf{0.7} & \textbf{0.8} & \textbf{1.4} & \textbf{13.4} \\
\bottomrule[0.15em]
\end{tabular}


\label{table:ablation_time}
\end{table}

\begin{table}[!t]
\centering
\caption{
    Comparison of training time for robots with different DoFs in a new environment.
}
\footnotesize
\begin{tabular}{lcc}
\toprule[0.15em]
\multirow{2}{*}{Methods} & \multicolumn{2}{c}{Time (h)} \\ \cmidrule(lr){2-3}
& 3-DoF robot & 6-DoF robot \\ \midrule
NTFields             & 1.0 & 9.1   \\ 
P-NTFields           & 3.7 & 31.6   \\ 
Ours                 & 1.3 & 14.1
                \\ \bottomrule[0.15em]
\end{tabular}

\label{table:ablation_training_time}
\end{table}

\subsection{Ablation Studies}

The effectiveness of the proposed adaptive planning strategy can be validated through \cref{table:3d_se2_se3,table:arm_simple}. As shown in \cref{table:ablation_points}, all pipelines utilizing physical constraints (with or without SADF) achieve a higher SR than the baseline NTFields, which verifies the effectiveness of our physical constraints.

We also investigate the influence of the SADF on our robot motion planning and present the quantitative results.
For our SADF, the robot's sparse and dense point clouds consist of 32 and 1024 points, respectively.
The Ground Truth (GT) to train our SADF is generated using BVH-distance-query \cite{Karras2012bvh} with 1024 sampled points of the robots (detailed in supplementary materials). Consequently, we evaluate our SADF against BVH-distance-query using 32 and 1024 sampled surface points within the time fields + physical constraints pipeline.
From \cref{table:ablation_points}, it is evident that utilizing our SADF results in a higher SR compared to using BVH-distance-query with only 32 surface points, and it is still competitive with using GT distance field generated with 1024 points.

At last,
we compare the collision-checking time cost of our SADF with BVH-distance-query and FCL (a common collision-checking library) \cite{pan2012fcl} to demonstrate the lightweight nature of our SADF as illustrated in \cref{table:ablation_time}. The results indicate that both FCL and BVH-distance-query are nearly 30 times slower than our learned SADF, particularly as the number of query points increases, which can be the bottleneck in both the training procedure and the real-time planning scenarios.

\section{Limitations and Future Work}

We report the training time for NTFields, P-NTFields, and our method in a new environment, as shown in \cref{table:ablation_training_time}. Our method shows comparable training times to NTFields and a significant reduction compared to P-NTFields.  Once trained, our method is able to generate paths in static scenes notably faster than training-free methods like RRT* and RRT-Connect.  However, in dynamic environments, training-free methods, which provide solutions within seconds, become valuable alternatives. A promising direction for future work is to further investigate the generalization capabilities of physics-informed methods in new environments.

Additionally, although the proposed neural SADF is agnostic to robot shapes, it is environment-specific, \ie, we need to train the SADF for a new environment. It is also interesting to further make it environment-agnostic.

\begin{figure*}
    \centering
    \includegraphics[width=\linewidth]{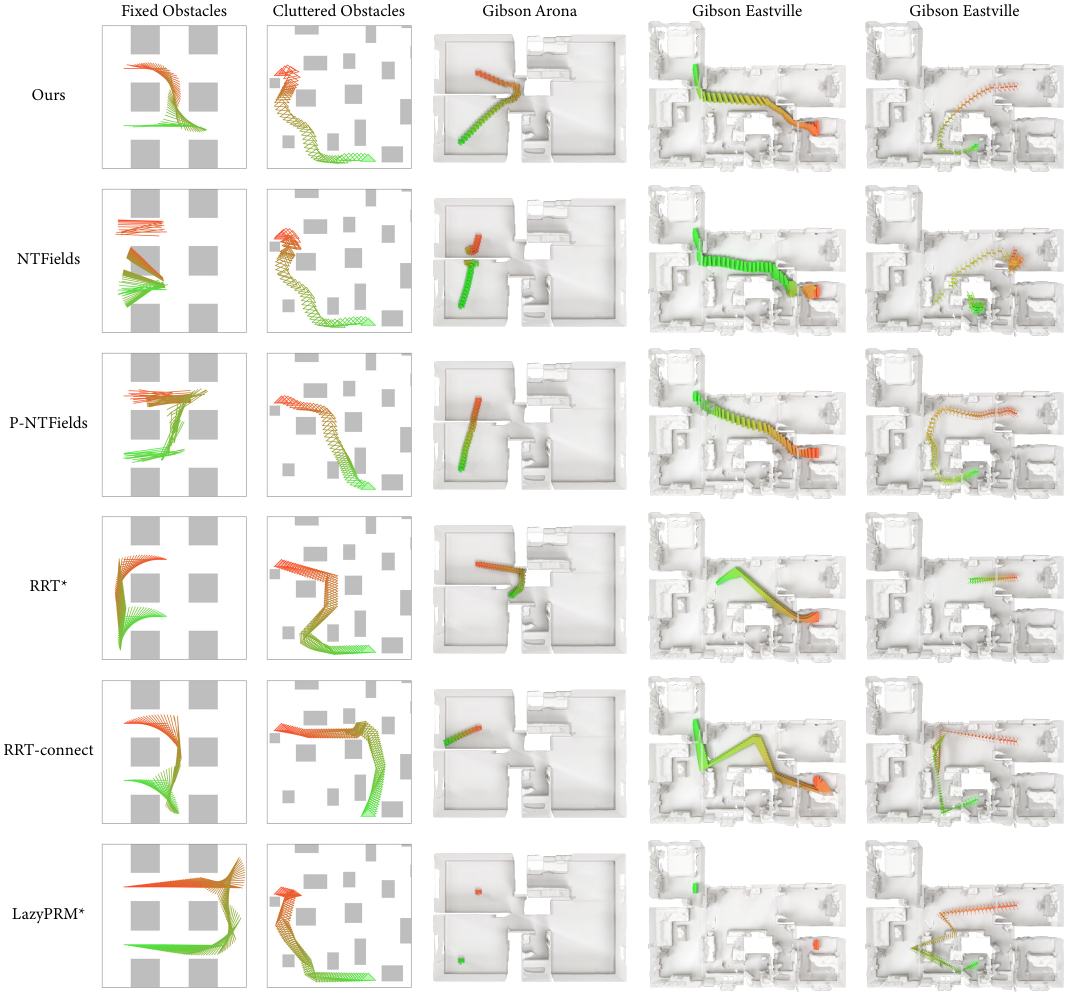}
    \caption{The comparison results on multiple environments with different robots of various shapes.}
    \label{fig:result}
    \Description{Planning result.}
\end{figure*}

\section{Conclusion}
This paper introduces a physics-constrained self-supervised learning framework for efficient and robust neural motion planning navigation in complex environments, named \pcplanner. To this end, we propose two physical constraints, namely monotonic and optimal constraints, to mitigate issues related to local minima in the Eikonal equation and introduce a novel shape-aware distance field to expedite the application of the physical constraints.
Additionally, we develop an adaptive motion planning strategy to enhance the robustness of our proposed \pcplanner.
Experiments with diverse robots in various scenarios demonstrate the efficacy of our method.

\begin{acks}
    We thank all the reviewers for their constructive comments and extend our gratitude to Xiao Liang for his help. We would also like to acknowledge the support of NSFC (No.~62102356, No.~62322207), Information Technology Center, and State Key Lab of CAD\&CG, Zhejiang University.
\end{acks}

\bibliographystyle{ACM-Reference-Format}
\bibliography{main}


\begin{thebibliography}{45}


\ifx \showCODEN    \undefined \def \showCODEN     #1{\unskip}     \fi
\ifx \showDOI      \undefined \def \showDOI       #1{#1}\fi
\ifx \showISBNx    \undefined \def \showISBNx     #1{\unskip}     \fi
\ifx \showISBNxiii \undefined \def \showISBNxiii  #1{\unskip}     \fi
\ifx \showISSN     \undefined \def \showISSN      #1{\unskip}     \fi
\ifx \showLCCN     \undefined \def \showLCCN      #1{\unskip}     \fi
\ifx \shownote     \undefined \def \shownote      #1{#1}          \fi
\ifx \showarticletitle \undefined \def \showarticletitle #1{#1}   \fi
\ifx \showURL      \undefined \def \showURL       {\relax}        \fi
\providecommand\bibfield[2]{#2}
\providecommand\bibinfo[2]{#2}
\providecommand\natexlab[1]{#1}
\providecommand\showeprint[2][]{arXiv:#2}

\bibitem[Barles and Souganidis(1991)]%
        {barles1991convergence}
\bibfield{author}{\bibinfo{person}{Guy Barles} {and}
  \bibinfo{person}{Panagiotis~E Souganidis}.} \bibinfo{year}{1991}\natexlab{}.
\newblock \showarticletitle{Convergence of approximation schemes for fully
  nonlinear second order equations}.
\newblock \bibinfo{journal}{\emph{Asymptotic analysis}} \bibinfo{volume}{4},
  \bibinfo{number}{3} (\bibinfo{year}{1991}), \bibinfo{pages}{271--283}.
\newblock


\bibitem[Bohlin and Kavraki(2000)]%
        {bohlin2000path}
\bibfield{author}{\bibinfo{person}{Robert Bohlin} {and}
  \bibinfo{person}{Lydia~E Kavraki}.} \bibinfo{year}{2000}\natexlab{}.
\newblock \showarticletitle{Path planning using lazy PRM}. In
  \bibinfo{booktitle}{\emph{Proceedings 2000 ICRA. Millennium conference. IEEE
  international conference on robotics and automation. Symposia proceedings
  (Cat. No. 00CH37065)}}, Vol.~\bibinfo{volume}{1}. IEEE,
  \bibinfo{pages}{521--528}.
\newblock


\bibitem[Chabra et~al\mbox{.}(2020)]%
        {chabra2020deep}
\bibfield{author}{\bibinfo{person}{Rohan Chabra}, \bibinfo{person}{Jan~E
  Lenssen}, \bibinfo{person}{Eddy Ilg}, \bibinfo{person}{Tanner Schmidt},
  \bibinfo{person}{Julian Straub}, \bibinfo{person}{Steven Lovegrove}, {and}
  \bibinfo{person}{Richard Newcombe}.} \bibinfo{year}{2020}\natexlab{}.
\newblock \showarticletitle{Deep local shapes: Learning local sdf priors for
  detailed 3d reconstruction}. In \bibinfo{booktitle}{\emph{Computer
  Vision--ECCV 2020: 16th European Conference, Glasgow, UK, August 23--28,
  2020, Proceedings, Part XXIX 16}}. Springer, \bibinfo{pages}{608--625}.
\newblock


\bibitem[Chang et~al\mbox{.}(2015)]%
        {chang2015shapenet}
\bibfield{author}{\bibinfo{person}{Angel~X Chang}, \bibinfo{person}{Thomas
  Funkhouser}, \bibinfo{person}{Leonidas Guibas}, \bibinfo{person}{Pat
  Hanrahan}, \bibinfo{person}{Qixing Huang}, \bibinfo{person}{Zimo Li},
  \bibinfo{person}{Silvio Savarese}, \bibinfo{person}{Manolis Savva},
  \bibinfo{person}{Shuran Song}, \bibinfo{person}{Hao Su}, {et~al\mbox{.}}}
  \bibinfo{year}{2015}\natexlab{}.
\newblock \showarticletitle{Shapenet: An information-rich 3d model repository}.
\newblock \bibinfo{journal}{\emph{arXiv preprint arXiv:1512.03012}}
  (\bibinfo{year}{2015}).
\newblock


\bibitem[Chaplot et~al\mbox{.}(2021)]%
        {chaplot2021differentiable}
\bibfield{author}{\bibinfo{person}{Devendra~Singh Chaplot},
  \bibinfo{person}{Deepak Pathak}, {and} \bibinfo{person}{Jitendra Malik}.}
  \bibinfo{year}{2021}\natexlab{}.
\newblock \showarticletitle{Differentiable spatial planning using
  transformers}. In \bibinfo{booktitle}{\emph{International Conference on
  Machine Learning}}. PMLR, \bibinfo{pages}{1484--1495}.
\newblock


\bibitem[Chen et~al\mbox{.}(2021)]%
        {chen2021equivariant}
\bibfield{author}{\bibinfo{person}{Haiwei Chen}, \bibinfo{person}{Shichen Liu},
  \bibinfo{person}{Weikai Chen}, \bibinfo{person}{Hao Li}, {and}
  \bibinfo{person}{Randall Hill}.} \bibinfo{year}{2021}\natexlab{}.
\newblock \showarticletitle{Equivariant point network for 3d point cloud
  analysis}. In \bibinfo{booktitle}{\emph{Proceedings of the IEEE/CVF
  conference on computer vision and pattern recognition}}.
  \bibinfo{pages}{14514--14523}.
\newblock


\bibitem[Chopp(2001)]%
        {chopp2001some}
\bibfield{author}{\bibinfo{person}{David~L Chopp}.}
  \bibinfo{year}{2001}\natexlab{}.
\newblock \showarticletitle{Some improvements of the fast marching method}.
\newblock \bibinfo{journal}{\emph{SIAM Journal on Scientific Computing}}
  \bibinfo{volume}{23}, \bibinfo{number}{1} (\bibinfo{year}{2001}),
  \bibinfo{pages}{230--244}.
\newblock


\bibitem[Chou et~al\mbox{.}(2022)]%
        {chou2022gensdf}
\bibfield{author}{\bibinfo{person}{Gene Chou}, \bibinfo{person}{Ilya Chugunov},
  {and} \bibinfo{person}{Felix Heide}.} \bibinfo{year}{2022}\natexlab{}.
\newblock \showarticletitle{GenSDF: Two-Stage Learning of Generalizable Signed
  Distance Functions}. In \bibinfo{booktitle}{\emph{Advances in Neural
  Information Processing Systems}},
  \bibfield{editor}{\bibinfo{person}{S.~Koyejo}, \bibinfo{person}{S.~Mohamed},
  \bibinfo{person}{A.~Agarwal}, \bibinfo{person}{D.~Belgrave},
  \bibinfo{person}{K.~Cho}, {and} \bibinfo{person}{A.~Oh}} (Eds.),
  Vol.~\bibinfo{volume}{35}. \bibinfo{publisher}{Curran Associates, Inc.},
  \bibinfo{pages}{24905--24919}.
\newblock
\urldef\tempurl%
\url{https://proceedings.neurips.cc/paper_files/paper/2022/file/9dfb5bc27e2d046199b38739e4ce64bd-Paper-Conference.pdf}
\showURL{%
\tempurl}


\bibitem[Crandall and Lions(1983)]%
        {crandall1983viscosity}
\bibfield{author}{\bibinfo{person}{Michael~G Crandall} {and}
  \bibinfo{person}{Pierre-Louis Lions}.} \bibinfo{year}{1983}\natexlab{}.
\newblock \showarticletitle{Viscosity solutions of Hamilton-Jacobi equations}.
\newblock \bibinfo{journal}{\emph{Transactions of the American mathematical
  society}} \bibinfo{volume}{277}, \bibinfo{number}{1} (\bibinfo{year}{1983}),
  \bibinfo{pages}{1--42}.
\newblock


\bibitem[Eppner et~al\mbox{.}(2020)]%
        {acronym2020}
\bibfield{author}{\bibinfo{person}{Clemens Eppner}, \bibinfo{person}{Arsalan
  Mousavian}, {and} \bibinfo{person}{Dieter Fox}.}
  \bibinfo{year}{2020}\natexlab{}.
\newblock \showarticletitle{{ACRONYM}: A Large-Scale Grasp Dataset Based on
  Simulation}. In \bibinfo{booktitle}{\emph{2021 {IEEE} Int. Conf. on Robotics
  and Automation, {ICRA}}}.
\newblock


\bibitem[Gammell et~al\mbox{.}(2015)]%
        {gammell2015batch}
\bibfield{author}{\bibinfo{person}{Jonathan~D Gammell},
  \bibinfo{person}{Siddhartha~S Srinivasa}, {and} \bibinfo{person}{Timothy~D
  Barfoot}.} \bibinfo{year}{2015}\natexlab{}.
\newblock \showarticletitle{Batch informed trees (BIT*): Sampling-based optimal
  planning via the heuristically guided search of implicit random geometric
  graphs}. In \bibinfo{booktitle}{\emph{2015 IEEE international conference on
  robotics and automation (ICRA)}}. IEEE, \bibinfo{pages}{3067--3074}.
\newblock


\bibitem[He et~al\mbox{.}(2016)]%
        {he2016deep}
\bibfield{author}{\bibinfo{person}{Kaiming He}, \bibinfo{person}{Xiangyu
  Zhang}, \bibinfo{person}{Shaoqing Ren}, {and} \bibinfo{person}{Jian Sun}.}
  \bibinfo{year}{2016}\natexlab{}.
\newblock \showarticletitle{Deep residual learning for image recognition}. In
  \bibinfo{booktitle}{\emph{Proceedings of the IEEE conference on computer
  vision and pattern recognition}}. \bibinfo{pages}{770--778}.
\newblock


\bibitem[Huh et~al\mbox{.}(2021)]%
        {huh2021cost}
\bibfield{author}{\bibinfo{person}{Jinwook Huh}, \bibinfo{person}{Volkan
  Isler}, {and} \bibinfo{person}{Daniel~D Lee}.}
  \bibinfo{year}{2021}\natexlab{}.
\newblock \showarticletitle{Cost-to-go function generating networks for high
  dimensional motion planning}. In \bibinfo{booktitle}{\emph{2021 IEEE
  International Conference on Robotics and Automation (ICRA)}}. IEEE,
  \bibinfo{pages}{8480--8486}.
\newblock


\bibitem[Ichter et~al\mbox{.}(2018)]%
        {ichter2018learning}
\bibfield{author}{\bibinfo{person}{Brian Ichter}, \bibinfo{person}{James
  Harrison}, {and} \bibinfo{person}{Marco Pavone}.}
  \bibinfo{year}{2018}\natexlab{}.
\newblock \showarticletitle{Learning sampling distributions for robot motion
  planning}. In \bibinfo{booktitle}{\emph{2018 IEEE International Conference on
  Robotics and Automation (ICRA)}}. IEEE, \bibinfo{pages}{7087--7094}.
\newblock


\bibitem[Jiang et~al\mbox{.}(2020)]%
        {jiang2020local}
\bibfield{author}{\bibinfo{person}{Chiyu Jiang}, \bibinfo{person}{Avneesh Sud},
  \bibinfo{person}{Ameesh Makadia}, \bibinfo{person}{Jingwei Huang},
  \bibinfo{person}{Matthias Nie{\ss}ner}, \bibinfo{person}{Thomas Funkhouser},
  {et~al\mbox{.}}} \bibinfo{year}{2020}\natexlab{}.
\newblock \showarticletitle{Local implicit grid representations for 3d scenes}.
  In \bibinfo{booktitle}{\emph{Proceedings of the IEEE/CVF Conference on
  Computer Vision and Pattern Recognition}}. \bibinfo{pages}{6001--6010}.
\newblock


\bibitem[Johnson et~al\mbox{.}(2023)]%
        {LSDtransformer}
\bibfield{author}{\bibinfo{person}{Jacob~J. Johnson}, \bibinfo{person}{Ahmed~H.
  Qureshi}, {and} \bibinfo{person}{Michael~C. Yip}.}
  \bibinfo{year}{2023}\natexlab{}.
\newblock \showarticletitle{Learning Sampling Dictionaries for Efficient and
  Generalizable Robot Motion Planning With Transformers}.
\newblock \bibinfo{journal}{\emph{IEEE Robotics and Automation Letters}}
  \bibinfo{volume}{8}, \bibinfo{number}{12} (\bibinfo{year}{2023}),
  \bibinfo{pages}{7946--7953}.
\newblock
\urldef\tempurl%
\url{https://doi.org/10.1109/LRA.2023.3322087}
\showDOI{\tempurl}


\bibitem[Kalakrishnan et~al\mbox{.}(2011)]%
        {kalakrishnan2011stomp}
\bibfield{author}{\bibinfo{person}{Mrinal Kalakrishnan},
  \bibinfo{person}{Sachin Chitta}, \bibinfo{person}{Evangelos Theodorou},
  \bibinfo{person}{Peter Pastor}, {and} \bibinfo{person}{Stefan Schaal}.}
  \bibinfo{year}{2011}\natexlab{}.
\newblock \showarticletitle{STOMP: Stochastic trajectory optimization for
  motion planning}. In \bibinfo{booktitle}{\emph{2011 IEEE international
  conference on robotics and automation}}. IEEE, \bibinfo{pages}{4569--4574}.
\newblock


\bibitem[Karaman and Frazzoli(2011)]%
        {karaman2011sampling}
\bibfield{author}{\bibinfo{person}{Sertac Karaman} {and}
  \bibinfo{person}{Emilio Frazzoli}.} \bibinfo{year}{2011}\natexlab{}.
\newblock \showarticletitle{Sampling-based algorithms for optimal motion
  planning}.
\newblock \bibinfo{journal}{\emph{The international journal of robotics
  research}} \bibinfo{volume}{30}, \bibinfo{number}{7} (\bibinfo{year}{2011}),
  \bibinfo{pages}{846--894}.
\newblock


\bibitem[Karras(2012)]%
        {Karras2012bvh}
\bibfield{author}{\bibinfo{person}{Tero Karras}.}
  \bibinfo{year}{2012}\natexlab{}.
\newblock \showarticletitle{Maximizing Parallelism in the Construction of BVHs,
  Octrees, and K-d Trees}. In \bibinfo{booktitle}{\emph{Proceedings of the
  Fourth ACM SIGGRAPH / Eurographics Conference on High-Performance Graphics}}.
  \bibinfo{publisher}{Eurographics Association}, \bibinfo{pages}{33--37}.
\newblock
\urldef\tempurl%
\url{https://doi.org/10.2312/EGGH/HPG12/033-037}
\showDOI{\tempurl}


\bibitem[Kingston et~al\mbox{.}(2018)]%
        {kingston2018sampling}
\bibfield{author}{\bibinfo{person}{Zachary Kingston}, \bibinfo{person}{Mark
  Moll}, {and} \bibinfo{person}{Lydia~E Kavraki}.}
  \bibinfo{year}{2018}\natexlab{}.
\newblock \showarticletitle{Sampling-based methods for motion planning with
  constraints}.
\newblock \bibinfo{journal}{\emph{Annual review of control, robotics, and
  autonomous systems}}  \bibinfo{volume}{1} (\bibinfo{year}{2018}),
  \bibinfo{pages}{159--185}.
\newblock


\bibitem[Kuffner and LaValle(2000)]%
        {kuffner2000rrt}
\bibfield{author}{\bibinfo{person}{James~J Kuffner} {and}
  \bibinfo{person}{Steven~M LaValle}.} \bibinfo{year}{2000}\natexlab{}.
\newblock \showarticletitle{RRT-connect: An efficient approach to single-query
  path planning}. In \bibinfo{booktitle}{\emph{Proceedings 2000 ICRA.
  Millennium Conference. IEEE International Conference on Robotics and
  Automation. Symposia Proceedings (Cat. No. 00CH37065)}},
  Vol.~\bibinfo{volume}{2}. IEEE, \bibinfo{pages}{995--1001}.
\newblock


\bibitem[Kurenkov et~al\mbox{.}(2022)]%
        {kurenkov2022nfomp}
\bibfield{author}{\bibinfo{person}{Mikhail Kurenkov}, \bibinfo{person}{Andrei
  Potapov}, \bibinfo{person}{Alena Savinykh}, \bibinfo{person}{Evgeny Yudin},
  \bibinfo{person}{Evgeny Kruzhkov}, \bibinfo{person}{Pavel Karpyshev}, {and}
  \bibinfo{person}{Dzmitry Tsetserukou}.} \bibinfo{year}{2022}\natexlab{}.
\newblock \showarticletitle{NFOMP: Neural Field for Optimal Motion Planner of
  Differential Drive Robots With Nonholonomic Constraints}.
\newblock \bibinfo{journal}{\emph{IEEE Robotics and Automation Letters}}
  \bibinfo{volume}{7}, \bibinfo{number}{4} (\bibinfo{year}{2022}),
  \bibinfo{pages}{10991--10998}.
\newblock


\bibitem[LaValle and Kuffner~Jr(2001)]%
        {lavalle2001randomized}
\bibfield{author}{\bibinfo{person}{Steven~M LaValle} {and}
  \bibinfo{person}{James~J Kuffner~Jr}.} \bibinfo{year}{2001}\natexlab{}.
\newblock \showarticletitle{Randomized kinodynamic planning}.
\newblock \bibinfo{journal}{\emph{The international journal of robotics
  research}} \bibinfo{volume}{20}, \bibinfo{number}{5} (\bibinfo{year}{2001}),
  \bibinfo{pages}{378--400}.
\newblock


\bibitem[Li et~al\mbox{.}(2021)]%
        {li2021learning}
\bibfield{author}{\bibinfo{person}{Xueting Li}, \bibinfo{person}{Shalini
  De~Mello}, \bibinfo{person}{Xiaolong Wang}, \bibinfo{person}{Ming-Hsuan
  Yang}, \bibinfo{person}{Jan Kautz}, {and} \bibinfo{person}{Sifei Liu}.}
  \bibinfo{year}{2021}\natexlab{}.
\newblock \showarticletitle{Learning continuous environment fields via implicit
  functions}.
\newblock \bibinfo{journal}{\emph{arXiv preprint arXiv:2111.13997}}
  (\bibinfo{year}{2021}).
\newblock


\bibitem[Li et~al\mbox{.}(2024)]%
        {li2024representing}
\bibfield{author}{\bibinfo{person}{Yiming Li}, \bibinfo{person}{Yan Zhang},
  \bibinfo{person}{Amirreza Razmjoo}, {and} \bibinfo{person}{Sylvain Calinon}.}
  \bibinfo{year}{2024}\natexlab{}.
\newblock \showarticletitle{Representing Robot Geometry as Distance Fields:
  Applications to Whole-body Manipulation}. In \bibinfo{booktitle}{\emph{Proc.
  IEEE ICRA}}.
\newblock


\bibitem[Loshchilov and Hutter(2019)]%
        {loshchilov2017decoupled}
\bibfield{author}{\bibinfo{person}{Ilya Loshchilov} {and}
  \bibinfo{person}{Frank Hutter}.} \bibinfo{year}{2019}\natexlab{}.
\newblock \showarticletitle{Decoupled Weight Decay Regularization}. In
  \bibinfo{booktitle}{\emph{7th International Conference on Learning
  Representations, {ICLR} 2019, New Orleans, LA, USA, May 6-9, 2019}}.
  \bibinfo{publisher}{OpenReview.net}.
\newblock
\urldef\tempurl%
\url{https://openreview.net/forum?id=Bkg6RiCqY7}
\showURL{%
\tempurl}


\bibitem[Mukadam et~al\mbox{.}(2016)]%
        {mukadam2016gaussian}
\bibfield{author}{\bibinfo{person}{Mustafa Mukadam}, \bibinfo{person}{Xinyan
  Yan}, {and} \bibinfo{person}{Byron Boots}.} \bibinfo{year}{2016}\natexlab{}.
\newblock \showarticletitle{Gaussian process motion planning}. In
  \bibinfo{booktitle}{\emph{2016 IEEE international conference on robotics and
  automation (ICRA)}}. IEEE, \bibinfo{pages}{9--15}.
\newblock


\bibitem[Ni and Qureshi(2023a)]%
        {ni2023ntfields}
\bibfield{author}{\bibinfo{person}{Ruiqi Ni} {and} \bibinfo{person}{Ahmed~H
  Qureshi}.} \bibinfo{year}{2023}\natexlab{a}.
\newblock \showarticletitle{{NTF}ields: Neural Time Fields for Physics-Informed
  Robot Motion Planning}. In \bibinfo{booktitle}{\emph{International Conference
  on Learning Representations}}.
\newblock
\urldef\tempurl%
\url{https://openreview.net/forum?id=ApF0dmi1_9K}
\showURL{%
\tempurl}


\bibitem[Ni and Qureshi(2023b)]%
        {ni2023progressive}
\bibfield{author}{\bibinfo{person}{Ruiqi Ni} {and} \bibinfo{person}{Ahmed~H
  Qureshi}.} \bibinfo{year}{2023}\natexlab{b}.
\newblock \showarticletitle{Progressive Learning for Physics-informed Neural
  Motion Planning}.
\newblock \bibinfo{journal}{\emph{arXiv preprint arXiv:2306.00616}}
  (\bibinfo{year}{2023}).
\newblock


\bibitem[Ni and Qureshi(2024)]%
        {ni2024physics}
\bibfield{author}{\bibinfo{person}{Ruiqi Ni} {and} \bibinfo{person}{Ahmed~H
  Qureshi}.} \bibinfo{year}{2024}\natexlab{}.
\newblock \showarticletitle{Physics-informed Neural Motion Planning on
  Constraint Manifolds}.
\newblock \bibinfo{journal}{\emph{arXiv preprint arXiv:2403.05765}}
  (\bibinfo{year}{2024}).
\newblock


\bibitem[Ouasfi and Boukhayma(2022)]%
        {ouasfi2022few}
\bibfield{author}{\bibinfo{person}{Amine Ouasfi} {and} \bibinfo{person}{Adnane
  Boukhayma}.} \bibinfo{year}{2022}\natexlab{}.
\newblock \showarticletitle{Few ‘zero level set’-shot learning of shape
  signed distance functions in feature space}. In
  \bibinfo{booktitle}{\emph{European Conference on Computer Vision}}. Springer,
  \bibinfo{pages}{561--578}.
\newblock


\bibitem[Pan et~al\mbox{.}(2012)]%
        {pan2012fcl}
\bibfield{author}{\bibinfo{person}{Jia Pan}, \bibinfo{person}{Sachin Chitta},
  {and} \bibinfo{person}{Dinesh Manocha}.} \bibinfo{year}{2012}\natexlab{}.
\newblock \showarticletitle{FCL: A general purpose library for collision and
  proximity queries}. In \bibinfo{booktitle}{\emph{2012 IEEE International
  Conference on Robotics and Automation}}. IEEE, \bibinfo{pages}{3859--3866}.
\newblock


\bibitem[Park et~al\mbox{.}(2019)]%
        {park2019deepsdf}
\bibfield{author}{\bibinfo{person}{Jeong~Joon Park}, \bibinfo{person}{Peter
  Florence}, \bibinfo{person}{Julian Straub}, \bibinfo{person}{Richard
  Newcombe}, {and} \bibinfo{person}{Steven Lovegrove}.}
  \bibinfo{year}{2019}\natexlab{}.
\newblock \showarticletitle{Deepsdf: Learning continuous signed distance
  functions for shape representation}. In \bibinfo{booktitle}{\emph{Proceedings
  of the IEEE/CVF conference on computer vision and pattern recognition}}.
  \bibinfo{pages}{165--174}.
\newblock


\bibitem[Ronneberger et~al\mbox{.}(2015)]%
        {ronneberger2015u}
\bibfield{author}{\bibinfo{person}{Olaf Ronneberger}, \bibinfo{person}{Philipp
  Fischer}, {and} \bibinfo{person}{Thomas Brox}.}
  \bibinfo{year}{2015}\natexlab{}.
\newblock \showarticletitle{U-net: Convolutional networks for biomedical image
  segmentation}. In \bibinfo{booktitle}{\emph{Medical image computing and
  computer-assisted intervention--MICCAI 2015: 18th international conference,
  Munich, Germany, October 5-9, 2015, proceedings, part III 18}}. Springer,
  \bibinfo{pages}{234--241}.
\newblock


\bibitem[Sethian(1996)]%
        {sethian1996fast}
\bibfield{author}{\bibinfo{person}{James~A Sethian}.}
  \bibinfo{year}{1996}\natexlab{}.
\newblock \showarticletitle{A fast marching level set method for monotonically
  advancing fronts.}
\newblock \bibinfo{journal}{\emph{proceedings of the National Academy of
  Sciences}} \bibinfo{volume}{93}, \bibinfo{number}{4} (\bibinfo{year}{1996}),
  \bibinfo{pages}{1591--1595}.
\newblock


\bibitem[Sethian(1999)]%
        {sethian1999fmm}
\bibfield{author}{\bibinfo{person}{J.~A. Sethian}.}
  \bibinfo{year}{1999}\natexlab{}.
\newblock \showarticletitle{Fast Marching Methods}.
\newblock \bibinfo{journal}{\emph{SIAM Rev.}} \bibinfo{volume}{41},
  \bibinfo{number}{2} (\bibinfo{year}{1999}), \bibinfo{pages}{199--235}.
\newblock
\urldef\tempurl%
\url{https://doi.org/10.1137/S0036144598347059}
\showDOI{\tempurl}


\bibitem[Treister and Haber(2016)]%
        {treister2016fast}
\bibfield{author}{\bibinfo{person}{Eran Treister} {and} \bibinfo{person}{Eldad
  Haber}.} \bibinfo{year}{2016}\natexlab{}.
\newblock \showarticletitle{A fast marching algorithm for the factored eikonal
  equation}.
\newblock \bibinfo{journal}{\emph{Journal of Computational physics}}
  \bibinfo{volume}{324} (\bibinfo{year}{2016}), \bibinfo{pages}{210--225}.
\newblock


\bibitem[Tugurlan(2008)]%
        {tugurlan2008fast}
\bibfield{author}{\bibinfo{person}{Maria~Cristina Tugurlan}.}
  \bibinfo{year}{2008}\natexlab{}.
\newblock \bibinfo{booktitle}{\emph{Fast marching methods-parallel
  implementation and analysis}}.
\newblock \bibinfo{publisher}{Louisiana State University and Agricultural \&
  Mechanical College}.
\newblock


\bibitem[Tukan et~al\mbox{.}(2022)]%
        {tukan2022obstacle}
\bibfield{author}{\bibinfo{person}{Murad Tukan}, \bibinfo{person}{Alaa
  Maalouf}, \bibinfo{person}{Dan Feldman}, {and} \bibinfo{person}{Roi
  Poranne}.} \bibinfo{year}{2022}\natexlab{}.
\newblock \showarticletitle{Obstacle aware sampling for path planning}. In
  \bibinfo{booktitle}{\emph{2022 IEEE/RSJ International Conference on
  Intelligent Robots and Systems (IROS)}}. IEEE, \bibinfo{pages}{13676--13683}.
\newblock


\bibitem[Vysock{\`y} et~al\mbox{.}(2019)]%
        {vysocky2019motion}
\bibfield{author}{\bibinfo{person}{Ale{\v{s}} Vysock{\`y}},
  \bibinfo{person}{Hisaka Wada}, \bibinfo{person}{Jun Kinugawa}, {and}
  \bibinfo{person}{Kazuhiro Kosuge}.} \bibinfo{year}{2019}\natexlab{}.
\newblock \showarticletitle{Motion planning analysis according to ISO/TS 15066
  in human--robot collaboration environment}. In \bibinfo{booktitle}{\emph{2019
  IEEE/ASME International Conference on Advanced Intelligent Mechatronics
  (AIM)}}. IEEE, \bibinfo{pages}{151--156}.
\newblock


\bibitem[Wang et~al\mbox{.}(2020)]%
        {Neural_RRT}
\bibfield{author}{\bibinfo{person}{Jiankun Wang}, \bibinfo{person}{Wenzheng
  Chi}, \bibinfo{person}{Chenming Li}, \bibinfo{person}{Chaoqun Wang}, {and}
  \bibinfo{person}{Max Q.-H. Meng}.} \bibinfo{year}{2020}\natexlab{}.
\newblock \showarticletitle{Neural RRT*: Learning-Based Optimal Path Planning}.
\newblock \bibinfo{journal}{\emph{IEEE Transactions on Automation Science and
  Engineering}} \bibinfo{volume}{17}, \bibinfo{number}{4}
  (\bibinfo{year}{2020}), \bibinfo{pages}{1748--1758}.
\newblock
\urldef\tempurl%
\url{https://doi.org/10.1109/TASE.2020.2976560}
\showDOI{\tempurl}


\bibitem[White et~al\mbox{.}(2020)]%
        {pykonal_fmm}
\bibfield{author}{\bibinfo{person}{Malcolm C.~A. White},
  \bibinfo{person}{Hongjian Fang}, \bibinfo{person}{Nori Nakata}, {and}
  \bibinfo{person}{Yehuda Ben‐Zion}.} \bibinfo{year}{2020}\natexlab{}.
\newblock \showarticletitle{{PyKonal: A Python Package for Solving the Eikonal
  Equation in Spherical and Cartesian Coordinates Using the Fast Marching
  Method}}.
\newblock \bibinfo{journal}{\emph{Seismological Research Letters}}
  \bibinfo{volume}{91}, \bibinfo{number}{4} (\bibinfo{date}{06}
  \bibinfo{year}{2020}), \bibinfo{pages}{2378--2389}.
\newblock


\bibitem[Xia et~al\mbox{.}(2018)]%
        {xiazamirhe2018gibsonenv}
\bibfield{author}{\bibinfo{person}{Fei Xia}, \bibinfo{person}{Amir R.~Zamir},
  \bibinfo{person}{Zhiyang He}, \bibinfo{person}{Alexander Sax},
  \bibinfo{person}{Jitendra Malik}, {and} \bibinfo{person}{Silvio Savarese}.}
  \bibinfo{year}{2018}\natexlab{}.
\newblock \showarticletitle{Gibson {Env}: real-world perception for embodied
  agents}. In \bibinfo{booktitle}{\emph{Computer Vision and Pattern Recognition
  (CVPR), 2018 IEEE Conference on}}. IEEE.
\newblock


\bibitem[Yang et~al\mbox{.}(2019)]%
        {yang2019survey}
\bibfield{author}{\bibinfo{person}{Yajue Yang}, \bibinfo{person}{Jia Pan},
  {and} \bibinfo{person}{Weiwei Wan}.} \bibinfo{year}{2019}\natexlab{}.
\newblock \showarticletitle{Survey of optimal motion planning}.
\newblock \bibinfo{journal}{\emph{IET Cyber-systems and Robotics}}
  \bibinfo{volume}{1}, \bibinfo{number}{1} (\bibinfo{year}{2019}),
  \bibinfo{pages}{13--19}.
\newblock


\bibitem[Zhu et~al\mbox{.}(2023)]%
        {zhu2023e2pn}
\bibfield{author}{\bibinfo{person}{Minghan Zhu}, \bibinfo{person}{Maani
  Ghaffari}, \bibinfo{person}{William~A Clark}, {and} \bibinfo{person}{Huei
  Peng}.} \bibinfo{year}{2023}\natexlab{}.
\newblock \showarticletitle{E2PN: Efficient SE (3)-equivariant point network}.
  In \bibinfo{booktitle}{\emph{Proceedings of the IEEE/CVF Conference on
  Computer Vision and Pattern Recognition}}. \bibinfo{pages}{1223--1232}.
\newblock


\end{thebibliography}

\clearpage

\appendix

\twocolumn[
\begin{flushleft}
\begin{spacing}{1.8}
{   \Huge \sffamily
    \pcplanner: Physics-Constrained Self-Supervised Learning for Robust
    Neural Motion Planning with Shape-Aware Distance Function
    Supplementary Material
}
\end{spacing}
\end{flushleft}
]

\setcounter{equation}{18}
\section{Factorized Eikonal Equation}

To make it applicable for motion planning tasks, we follow the factorization in NTFields \cite{ni2023ntfields} for $T(\cstart,\cgoal)$:
\begin{equation}
    T(\cstart,\cgoal) = \frac{\left \|\cstart-\cgoal  \right \|}{\tau(\cstart,\cgoal)} \label{T_define},
\end{equation}
where $\tau(\cstart,\cgoal)$ is a factorized time field. The advantage here is that $\tau$ can effectively adjust the $T(\cstart,\cgoal)\in \left [ 0,\infty  \right ] $ within a constrained range~(specifically from 0 to 1) and avoid the singularity issue.

We then employ the MLP to model the underlying physics of $\tau$. The neural network, parameterized by $\Theta$, takes the robot's start and goal configuration $(\cstart, \cgoal)$ as input and outputs the factorized time field $\tau_{\Theta}$:
\begin{equation}
    \label{eq:tau}
    \tau_{\Theta} =\text{MLP}(\cstart,\cgoal;\Theta).
\end{equation}
Subsequently, we can obtain the predicted time field $T_{\Theta}$ through \cref{eq:tau}.

\section{Implementation Details}

\subsection{Physics-Constrained Time Field}

\paragraph{Training strategy.}
The speed loss is applied throughout the entire training process, while the physical constraint loss is introduced after a certain number of epochs (specifically, 50 epochs). This delay is necessary because the time fields do not perform well in the initial training stages, and the waypoints derived at this stage will mislead the network's training. It is important to note that our monotonic constraint requires a relatively optimal waypoint to guide training effectively, and a poor-quality waypoint will disturb training of the time field.

\paragraph{Regressor Architecture.}
The network architecture of the regressor for the time field follows NTFields. It contains the c-space encoder, a non-linear symmetric operator, and the time field generator.
The c-space encoder, denoted as $g(\cdot)$, is comprised of fully connected~(FC) layers with ELU activation and several ResNet-style\cite{he2016deep} MLP with ELU, which takes the robot's configuration $\mathbf{c}$ as input and generate the embedding $g(\mathbf{c})$.
Given the start and goal configuration $\cstart$ and $\cgoal$, the non-linear symmetric operator is represented as $[\max(f(\cstart), f(\cgoal)), \min(f(\cstart), f(\cgoal))]$ where $[\cdot]$ denotes a concatenation operator. This operator ensures the output of the time field is symmetric with respect to the start and goal configuration.
The time field generator takes the concatenated embedding as input and outputs the time field value. It is composed of several FC + ELU layers and ResNet MLP + ELU layers.

\paragraph{Hyperparameters.}
We use AdamW~\cite{loshchilov2017decoupled} optimizer with $2\times10e^{-4}$ learning rate and 0.1 weight decay. During training, $\lambda_m$ and $\lambda_o$ are set to 0.08 and 0.001 respectively.

\subsection{Shape-Aware Distance Field}

\paragraph{Training Setup.}
We create the training dataset from Acronym \cite{acronym2020} following the procedure of GenSDF \cite{chou2022gensdf}. Acronym is a subset of ShapeNet \cite{chang2015shapenet} dataset which consists of 8872 watertight synthetic 3D models of 262 categories. We use 147 models as our training dataset to learn our SADF. For each 3D object, we sample 1024 points on the surface as a dense input, and then downsample to 32 points to create a sparse input.
For a given environment, we sample $10^6$ configurations for each object that we viewed as our robot, and then calculate the shape-aware distance by attaining the minimum of the queried distances from the 1024 sampled surface points to the environment with BVH-distance-query~\cite{Karras2012bvh}.
This distance obtained from the dense point cloud through BVH-distance-query is regarded as the GT for SADF training.

\paragraph{Training Details.}
We fix the parameters of the pretrained shape encoder in the training process.
For each epoch, we randomly select one object from the training dataset and use the shape encoder to derive the shape feature which will be only calculated once in this epoch. Then we derive the $\mathcal{L}_d$ and $\mathcal{L}_{SADF}$ to jointly optimize the network parameters.

\paragraph{Network Architecture.}
We use the pretrained shape decoder from GenSDF~\cite{chou2022gensdf} which chooses 256 latent size and 64 hidden dimensions. The shape encoder utilizes the tri-plane feature to represent the object's geometry and uses the parallel Unet \cite{ronneberger2015u} to aggregate shape information. For the shape decoder and distance decoder, we employ the fully connected~(FC) layers with ELU activation and several ResNet-style\cite{he2016deep} MLP with ELU.

\paragraph{Hyperparameters.}
We use AdamW~\cite{loshchilov2017decoupled} optimizer with $2\times10e^{-4}$ learning rate and 0.1 weight decay. During training, $\lambda_d$ is set to 1.

\section{More experiment details}
\label{sec:exp_details}

\subsection{Experimental Setup}
\label{sec:exp_setup}
\noindent
\textbf{Evaluation Metrics.}
Our evaluation metrics include \textbf{path length}, \textbf{planning time}, \textbf{success rate~(SR)} and \textbf{challenging success rate (CSR)}. The path length quantifies the sum of configuration distances between configurations of the waypoint in different settings, while the planning time measures the time taken by a planner to seek a valid path solution. The SR represents the percentage of collision-free paths connecting the provided start and goal in the test dataset identified by a given planner.
The CSR is built upon the SR which constructs a test dataset that removes the easy case. Here the easy case implies the trajectory can be successfully planned just using simple linear interpolation between the start and goal configuration.
The quantitative results for each set of experiments are averaged from the planning outcomes.

\noindent\textbf{Baselines.} We compare our methods with the following baselines.
\begin{itemize}[left=4pt]
    \item NTFields~\cite{ni2023ntfields}: As described earlier, it directly learns to solve the Eikonal equation without relying on expert training data.
    \item P-NTFields~\cite{ni2023progressive}: A physics-informed method based on NTFields that introduces a progressive learning strategy and incorporates a viscosity term into the Eikonal equation to deal with complex scenarios.
    \item FMM~\cite{sethian1996fast}: A numerical method~\cite{sethian1996fast} that discretizes the given C-space and computes the solution to the Eikonal equation for path planning.
    \item RRT*~\cite{kingston2018sampling}: A sampling-based method that constructs optimal trees and finds a feasible path connecting the given start and goal configuration.
    \item RRT-Connect~\cite{kuffner2000rrt}: A bidirectional sampling-based planner that iteratively grows trees from both the start and goal configurations, attempting to connect them by extending the trees towards each other until a path is found.
    \item Lazy-PRM*~\cite{bohlin2000path}: A multi-query method that combines sampling with graph search, which constructs a graph by sampling the environment and connecting nodes. When given start and goal configurations, it queries the graph to find a path.
\end{itemize}

\noindent\textbf{Experimental Settings.}
Following the data preparation procedure outlined in NTFields, we generate $10^6$(3,4 DoFs) or $10^7$ (6 DoFs) training configurations for NTFields, P-NTFields, and our method. For FMM, which involves the process of discretization, we opt to choose the nearest grid cells associated with our start and goal pairs when seeking solutions. 
Moreover, we execute RRT*, RRT-Connect, and LazyPRM* on our test set until they discover a path solution with the given start and goal configuration in the specified time limit~(10 seconds for rigid robots and 5 seconds for manipulators). 
For the success rate test, we randomly chose 1000 start and goal pairs that are collision-free as our test set.
All the experiments are conducted with 3090 RTX GPU and Intel(R) Xeon(R) Gold 6139M CPU.

\subsection{Motion Planning in 2D Environments.}
\label{sec:mp_2d_env}
We conduct a benchmark of our proposed method on two 2D environments in $\mathbb{SE}(2)$ space. The first environment consists of six obstacles with fixed sizes, while the second one comprises a cluster of 15 randomly placed obstacles with variable sizes.
We compare our method with the aforementioned baselines and demonstrate the planning capability with line and triangle as our robot in 2D environments.

\begin{table}[!t]
\centering
\caption{
    Comparison on the 2D environments in $\mathbb{SE}(2)$. The optimal results are highlighted with \colorbox{colorFst}{\bf first}, \colorbox{colorSnd}{second}.
}
\footnotesize

\setlength{\tabcolsep}{4pt}

\begin{tabular}{lccccc}
\toprule[0.15em]
\multirow{2}{*}{Methods}
& \multirow{2}{*}{Metrics} & \multicolumn{2}{c}{Fixed Obstacles} & \multicolumn{2}{c}{Cluttered Obstacles} \\ \cmidrule(lr){3-4} \cmidrule(lr){5-6}
&                          & Line             & Triangle         & Line               & Triangle         \\

\midrule[\ourmidrulewidth]
\multirow{\metricsize}{*}{RRT*}
& Length                   & 0.24             & \nd 0.31         & \nd 0.17           & 0.30              \\
& Time(ms)                 & 10140.1          & 10131.2          & 10191.9            & 10185.6           \\
& SR(\%)                   & 93.4             & 96.7             & 82.9               & 83.7              \\
& CSR(\%)                  & 81.8             & 89.8             & 59.7               & 62.9 
  \\
\midrule
\multirow{\metricsize}{*}{LazyPRM*}
& Length                   & \fs 0.22         & \fs 0.28     & \fs 0.16               & 0.27              \\
& Time(ms)                 & 10143.0          & 10112.1          & 10145.8            & 10142.1           \\
& SR(\%)                   & 94.7             & 96.7             & 88.7               & 89.7              \\
& CSR(\%)                  & 85.4             & 89.8             & 73.2               & 76.3              \\
\midrule
\multirow{\metricsize}{*}{RRT-Connect}
& Length                   & 0.68             & 0.79             & 0.64               & 0.79              \\
& Time(ms)                 & 612.1            & 262.7            & 1183.5             & 957.7             \\
& SR(\%)                   & 97.7             & 97.7             & 93.2               & 93.2              \\
& CSR(\%)                  & \nd 95.3         & 95.1             & 88.9               & 88.3               \\
\midrule[\ourmidrulewidth]
\multirow{\metricsize}{*}{FMM}
& Length                   & \nd 0.23         & \nd 0.31       & \fs 0.16     & \fs 0.25               \\
& Time(ms)                 & 1290.0           & 1381.2             & 1332.1             & 1762.3            \\
& SR(\%)                   & 98.4             & 99.2               & \nd 99.4           & 99.2               \\
& CSR(\%)                  & \fs 99.7         & 97.8               &  \fs 100           &  \fs 99.3              \\

\midrule
\multirow{\metricsize}{*}{NTFields}
& Length                   & 0.25             & \nd 0.31           & \nd 0.17           & \nd 0.26              \\
& Time(ms)                 & 3.3              & 2.4                & 5.3                & 3.1               \\
& SR(\%)                   & \nd 98.3         & \nd 99.9           & 82.6               & 94.0              \\
& CSR(\%)                  & \nd 95.3         & \nd 99.7               & 60.7               & 86.2               \\
\midrule
\multirow{\metricsize}{*}{P-NTFields}
& Length                   & 0.26             & 0.31               & \nd 0.17              & 0.26              \\
& Time(ms)                 & \nd 2.5          & \nd 2.3            & 4.3                   & \fs 2.0               \\
& SR(\%)                   & 96.9             & \fs 100            & 94.9                  & 97.2              \\
& CSR(\%)                  & 92.0             & \fs 100                & 84.1                  & 93.6               \\
\midrule[\ourmidrulewidth]

\multirow{\metricsize}{*}{\begin{tabular}[c]{@{}l@{}}Ours w/o\\ adapt. planning\end{tabular}}
& Length                   & 0.24             & \nd 0.31           & \fs 0.16               & \nd 0.26              \\
& Time(ms)                 & \fs 1.8          & \fs 2.0            & \fs 1.5                & \nd 2.2               \\
& SR(\%)                   & \fs 99.9         & \fs 100            & 99.3                   & 98.9              \\
& CSR(\%)                  & \fs 99.7         & \fs100              & 98.3                    & 97.5                \\
\midrule

\multirow{\metricsize}{*}{Ours}
& Length                   & 0.24             & \nd 0.31           & \fs 0.16               & \nd 0.26              \\
& Time(ms)                 & \fs 1.8          & \fs 2.0            & \nd 2.6                & 3.7               \\
& SR(\%)                   & \fs 99.9         & \fs 100            & \fs 99.8               & \fs 99.6              \\
& CSR(\%)                  & \fs 99.7         & \fs100               & \nd 99.5              & \nd 99.1               \\

\toprule[0.15em]
\end{tabular}

\label{table:2d_env}
\end{table}

From ~\cref{table:2d_env},  we can see that our method surpasses all other approaches in terms of both SR and computation time. Additionally, our method achieves almost the best CSR with only a slight margin behind FMM in the clustered environment. Notably, neural methods, including NTFields, P-NTFields, and our proposed method, demonstrate superior computational efficiency compared to traditional motion planning methods like RRT*, LazyPRM*, RRT-Connect, and FMM. 
Specifically, our method is more than 400 times faster than FMM, the numerical method to solve the Eikonal equation.
In 2D environments, our method outperforms traditional planning methods with a remarkable speed-up of over 100 times, while maintaining a comparable path length to other methods within a small threshold.

\subsection{Motion Planning on Constraint Manifolds}
We apply our methods to geodesic distance learning, framed as a constrained motion planning~(CMP) problem, on a bunny-shaped 2D surface mesh manifold in 3D space.
Following \cite{ni2024physics}, we define the GT speed model $S^{\ast }$ for the constrained motion planning problem as follows:

\begin{equation}
    \begin{aligned}
        D_{\mathcal{M}}(\mathbf{c}) &= \min\left(D_{[r, \mathcal{X}_{mnfld}]}\left(\mathbf{c} \right), d_{max} \right), \\
        S^{\ast }(\mathbf{c}) &=\exp (-\frac{D^{2}_{\mathcal{M}}(\mathbf{c})}{\beta d^2_{max}} ),
    \end{aligned}
    \label{speedmodel_mnfld}
\end{equation}
where $D_{\mathcal{M}}(\mathbf{c})$ determines the distance from the robot to the manifolds. $D$ is a function that computes the shortest distance between the robot $r$ at configuration $\mathbf{c}\in \mathcal{C}$ and manifold $\mathcal{X}_{mnfld}$, $d_{max}$ limits the maximum distance ranges, and $\beta \in \mathbb{R}^{+}$ is a scaling factor.
For further details about CMP with time fields, please refer to \cite{ni2024physics}.
As shown in \cref{fig:geodesic}, our method successfully finds an effective path solution (in green) to determine the geodesic distance.
\begin{figure}
\centering
    \begin{subfigure}[b]{0.32\columnwidth}
         \centering
	\includegraphics[width=\linewidth]{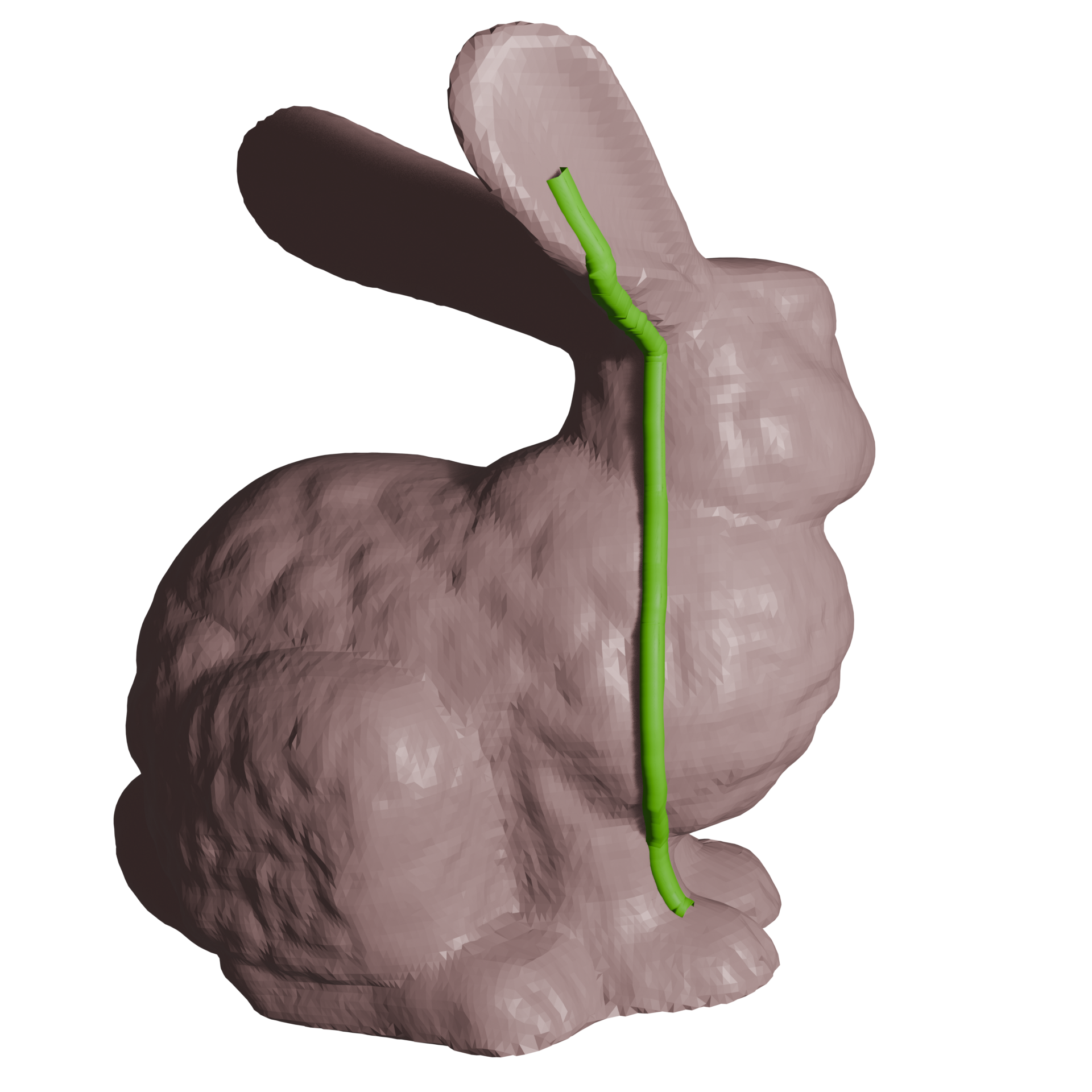}
    \end{subfigure}
    \begin{subfigure}[b]{0.32\columnwidth}
         \centering
	\includegraphics[width=\linewidth]{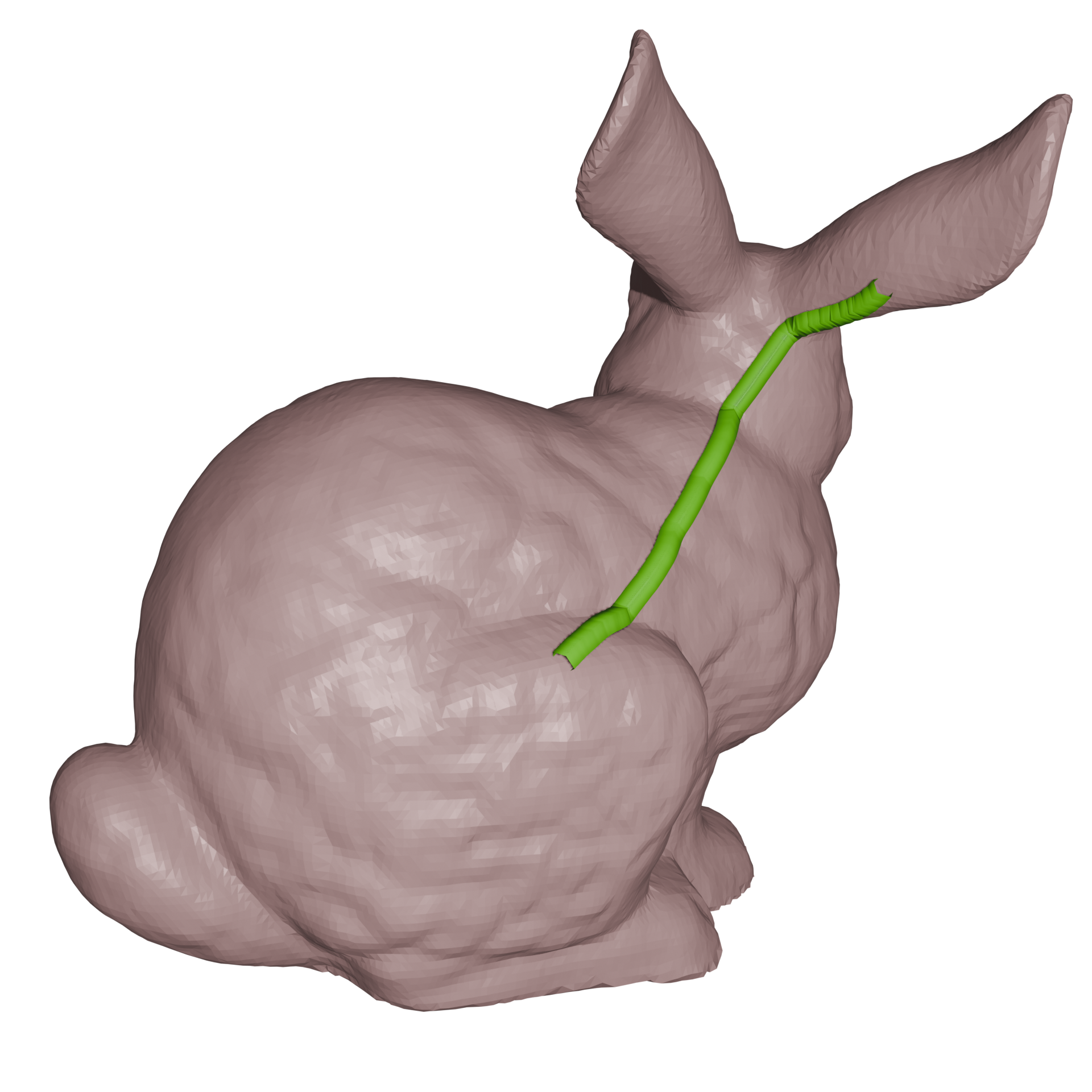} 
    \end{subfigure}
    \begin{subfigure}[b]{0.32\columnwidth}
         \centering
	\includegraphics[width=\linewidth]{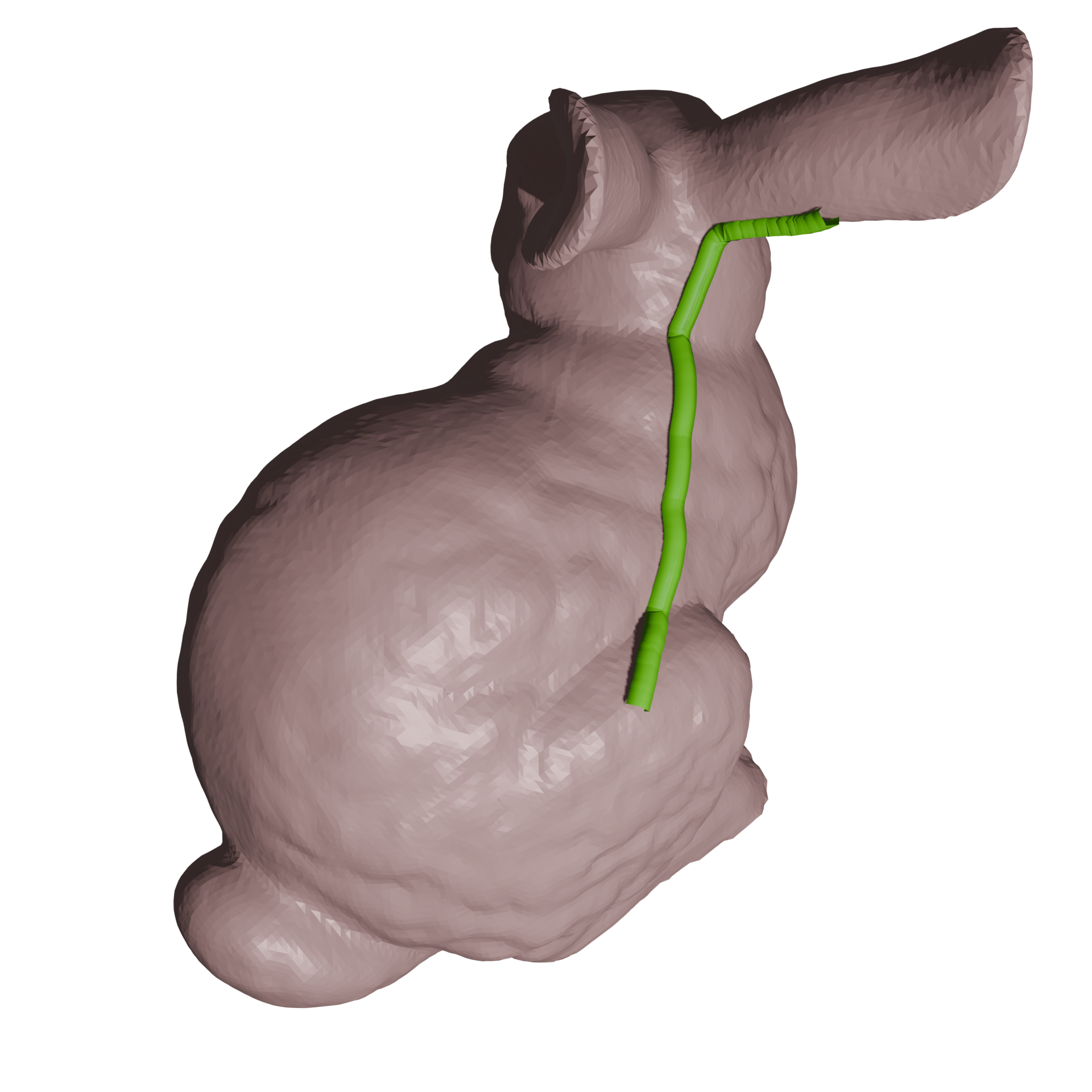} 
    \end{subfigure}
    \caption{
    	Constrained motion planning on 2D surface mesh manifold (Bunny) in 3D space, with the planned path highlighted in green.
    }
    \label{fig:geodesic}
    \Description{Constrained motion planning on Bunny.}
\end{figure}

\section{Limitation}
Our physical constraints serve as a necessary condition for ensuring the monotonicity of the time field in the viscosity solution of the Eikonal equation. However, these constraints are not sufficient to guarantee that the network will converge to the viscosity solution. Moreover, although the proposed neural shape-aware distance field (SADF) is independent of robot shapes, it is specific to the environment, requiring training for each new environment. Identifying the sufficient conditions for network convergence to the viscosity solution and developing an environment-agnostic SADF would be an intriguing direction for future research.

\end{document}